\documentclass[jair,twoside,11pt,theapa]{article}
\usepackage{jair, theapa, rawfonts}
\pdfoutput=1

\usepackage{xcolor}
\definecolor{Comment}{RGB}{1,80,50}
\usepackage[ruled]{algorithm2e}

\SetCommentSty{mycommfont}

\usepackage{svg}

\usepackage{amsmath}
\usepackage{amsfonts}
\usepackage{amssymb}
\usepackage{amsthm}

\DeclareMathOperator{\SMUS}{SMUS}

\usepackage{booktabs}
\usepackage{makecell}
\usepackage{url}

\usepackage{wrapfig}

\usepackage{caption}
\usepackage{subcaption}

\usepackage{tikz}
\usepackage{tikzscale}
\usepackage{tkz-graph}  
\usetikzlibrary{shapes,decorations,arrows,calc,arrows.meta,fit,positioning}
\tikzset{
    -Latex,auto,node distance =0.6 cm and 0.3 cm,semithick,
    state/.style ={rectangle, draw, minimum width = 0.5 cm},
    frontier-state/.style ={rectangle, draw, minimum width = 0.5 cm, fill=black, text=white},
    point/.style = {circle, draw, inner sep=0.04cm,fill,node contents={}},
    el/.style = {inner sep=2pt, align=left, sloped}
}

\theoremstyle{definition}
\newtheorem{definition}{Definition}
\newtheorem{example}{Example}
\theoremstyle{plain}
\newtheorem{lemma}{Lemma}
\newtheorem{theorem}{Theorem}
\newtheorem{proposition}{Proposition}
\newtheorem{corollary}{Corollary}
\theoremstyle{remark}
\newtheorem*{note}{Note}

\SetKw{Continue}{continue}

\ShortHeadings{How to discover short, shorter, and the shortest proofs of UNSAT formulas}
{Sidorov, Van der Linden, Correia, De Weerdt \& Demirović}

\begin{document}

\newcommand{\Previous}[1]{\mathtt{Previous}[#1]}
\newcommand{\Current}[1]{\mathtt{Current}[#1]}
\newcommand{\Next}[1]{\mathtt{Next}[#1]}
\newcommand{\Forgotten}[1]{\mathtt{Forgotten}[#1]}
\newcommand{\Known}[1]{\mathtt{Known}[#1]}

\newcommand{\Partition}[1]{\mathtt{Partition}(#1)}
\newcommand{\Bound}[1]{\mathtt{Bound}(#1)}
\newcommand{\Lookup}[2]{\mathtt{Lookup}(#1; #2)}
\newcommand{\Complete}[1]{\mathtt{Complete}(#1)}

\newcommand{\Frontier}[1]{\mathtt{Frontier}[#1]}
\newcommand{\fin}{\in_{\mathcal{F}}}
\newcommand{\fsubseteq}{\subseteq_{\mathcal{F}}}

\title{How to discover short, shorter, and the shortest proofs of unsatisfiability: a branch-and-bound approach for resolution proof length minimization}

\author{\name Konstantin Sidorov \email k.sidorov@tudelft.nl \\
       \name Koos van der Linden \email j.g.m.vanderLinden@tudelft.nl \\
       \addr EEMCS, Delft University of Technology, \\ Van Mourik Broekmanweg 6, 2628 XE Delft, The Netherlands
       \AND
       \name Gonçalo Homem de Almeida Correia \email G.Correia@tudelft.nl \\
       \addr Faculty of Civil Engineering and Geosciences, Delft University of Technology, \\ Stevinweg 1, 2628 CN Delft, The Netherlands
       \AND
       \name Mathijs de Weerdt \email m.m.deweerdt@tudelft.nl \\
       \name Emir Demirović \email e.demirovic@tudelft.nl \\
       \addr EEMCS, Delft University of Technology, \\ Van Mourik Broekmanweg 6, 2628 XE Delft, The Netherlands
       }

\maketitle

\begin{abstract}
Modern software for propositional satisfiability problems gives a powerful automated reasoning toolkit, capable of outputting not only a satisfiable/unsatisfiable signal but also a justification of unsatisfiability in the form of resolution proof (or a more expressive proof), which is commonly used for verification purposes. Empirically, modern SAT solvers produce relatively short proofs, however, there are no inherent guarantees that these proofs cannot be significantly reduced. This paper proposes a novel branch-and-bound algorithm for finding the shortest resolution proofs; to this end, we introduce a layer list representation of proofs that groups clauses by their level of indirection. As we show, this representation breaks all permutational symmetries, thereby improving upon the state-of-the-art symmetry-breaking and informing the design of a novel workflow for proof minimization. In addition to that, we design pruning procedures that reason on proof length lower bound, clause subsumption, and dominance. Our experiments suggest that the proofs from state-of-the-art solvers could be shortened by 30---60\% on the instances from SAT Competition 2002 and by 25---50\% on small synthetic formulas. When treated as an algorithm for finding \emph{the shortest} proof, our approach solves twice as many instances as the previous work based on SAT solving and reduces the time to optimality by orders of magnitude for the instances solved by both approaches.
\end{abstract}

\section{Introduction}

For the last two decades, the field of propositional satisfiability has experienced explosive growth, which has led to a wide range of applications of SAT solvers. Some of the common examples include model checking \shortcite{Prasad2005-intJSoftwToolsTechnolTransf}, AI planning \shortcite{Ernst1997-Other}, and combinatorial designs \cite{Zhang2021-Other}. A trait common to many domains relying on SAT solving is that unsatisfiable formulas have a meaningful domain interpretation (e.g., the correctness properties in verification domains), rather than merely being artifacts of poor modeling. However, as many of these domains have rigorous demands for the results produced by a solver, this introduces the following question: it is straightforward to check that there \emph{is} a solution once it is reported, but how to check the claim that there are \emph{none}?

\begin{wrapfigure}{r}{0.48\textwidth}
    \centering
    \includesvg[width=0.4\textwidth]{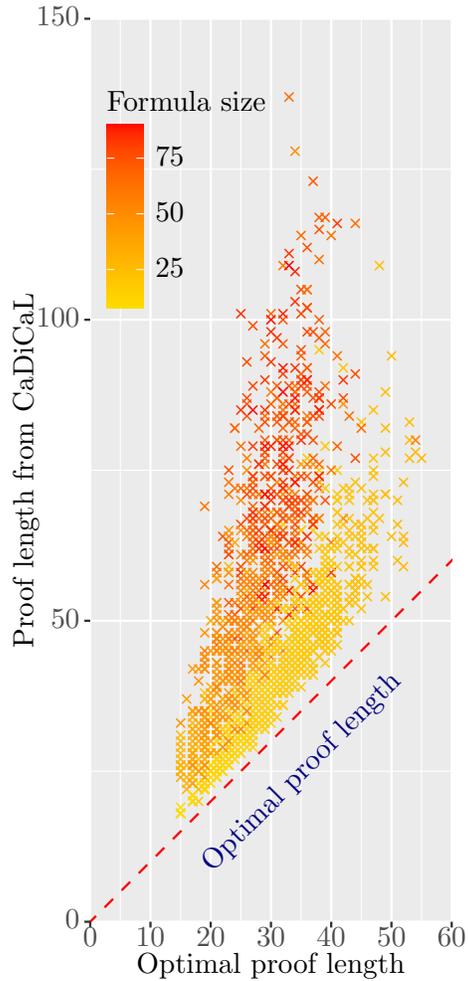}
    \caption{CaDiCaL proof length comparison against the shortest proofs on random unsatisfiable 3-CNFs.}
    \label{fig:optimal-vs-cadical}
\end{wrapfigure}

These developments created a demand for independent verification of the solver output. The standard way of achieving this is known as proof logging \cite{Heule2021-Other}, referring to the techniques that record the sequence of inferences a solver makes while establishing unsatisfiability. Established tools, such as DRAT-trim~\shortcite{Wetzler2014-Other}, then use this information to verify that the solver has made no incorrect inferences by ensuring that the proof log adheres to the standard imposed by the checker, thus playing a crucial role in building trust in the reliability of modern SAT solvers.

Thus, the role of a SAT solver now extends beyond reporting a SAT/UNSAT result; it can also be seen as an \emph{inference engine} within some pre-defined formal proof systems. Some common examples include resolution \cite{Davis1960-jAcm}, extended resolution \cite{Tseitin1983-Other}, or resolution asymmetric tautology \cite{Wetzler2014-Other}. The proof systems provide a formal underpinning to the proof-logging techniques, as they establish the inferences the solver can declare during the proof logging.

While SAT solvers are commonly used to derive proofs and empirically succeed in finding short proofs, the discovered proofs are not necessarily \emph{the shortest} available ones. To substantiate this point, we have compared the proofs for small random unsatisfiable 3-CNF formulas produced by \emph{CaDiCaL} \shortcite{Biere2024-Other}, a state-of-the-art SAT solver, with the shortest possible proofs for these formulas discovered with our approach. This comparison is summarized in Figure~\ref{fig:optimal-vs-cadical}; for formulas with \( \ge 30 \) clauses, CaDiCaL proofs are typically at least \( 50\% \) longer than the optimal ones, with many instances exhibiting at least twice the optimal length. These results reveal that the CaDiCaL proofs can be short yet far from being \emph{the shortest}. In computational terms, this result not only indicates the opportunity to improve CaDiCaL reasoning on these formulas but also estimates the ``room for improvement" in terms of the number of solver steps.

Given that, we can think about finding proofs of unsatisfiability as solving an optimization problem, where the feasible set is the set of all proofs starting from a given formula and the objective is the proof length. Unfortunately, discovering \emph{short} proofs is intuitively an exponentially harder problem than discovering \emph{valid} proofs---even a single proof may be exponentially larger than the input formula \shortcite{Haken1985-theorComputSci}---which is a likely reason why few works report successful proof minimization results. The seminal work by \citeA{Peitl2021-jArtifIntellRes} addresses this problem for resolution proofs: they propose a SAT encoding for valid proofs of a specific formula, which can be submitted to a SAT solver to either discover a resolution proof of a given length or show that none exists. This encoding, coupled with appropriate symmetry-breaking constraints, has been successfully used to establish upper bounds on the shortest proof length for a given number of clauses.

However, the aforementioned approach suffers from two significant limitations. First, as the search space for this problem grows with the size of the formula \emph{and} the length of proof, the empirical evaluation of this approach suggests that it is feasible only for formulas with up to a dozen clauses. Second, this approach enumerates lower bounds on the proof length until the proof of the target length is discovered, which is then declared optimal; in particular, at no time before the termination, is there any ``short" proof that could be returned, for example, when the allotted time runs out.

Their work also points out one of the major reasons why the proof minimization problem is difficult, namely, that the search space is highly symmetric, and proposes symmetry-breaking constraints to remedy this issue. While this technique has substantially improved the performance of a SAT solver on the proof minimization problems, we point out that this does not resolve all permutation symmetries. This observation suggests a direction for further algorithmic improvements that we pursue in our methodology. 

In this paper, we propose a novel branch-and-bound algorithm for minimizing the proof length that addresses all the aforementioned concerns. The key contributions supporting this approach are (a) a representation of the proof breaking \emph{all} symmetries stemming from clause permutations, (b) a procedure for deriving lower bounds on the proof length that generalizes the proposal from \citeA{Peitl2021-jArtifIntellRes} for the case of arbitrary formulas, which we use in our approach to stop searching for proofs once the bound becomes high enough, and (c) a dominance relation on proof prefixes that detects the proof steps that cannot improve upon the proofs explored earlier in the search.

Combined in a single workflow, this approach is capable of reducing the proof length of a state-of-the-art solver by 25\% to 50\% on synthetic formulas (e.g., 3-CNFs, graph coloring formulas), and by 30\% to 60\% on UNSAT formulas of \emph{SAT Competition 2002}. In addition, this algorithm substantially improves the \emph{complete} enumeration of proofs, as suggested by the comparisons with the encoding of \citeA{Peitl2021-jArtifIntellRes}, where our approach solves twice as many instances and terminates faster by at least an order of magnitude in the majority of solved instances. We also discover a limitation of our approach that is connected to the memory consumption of the resolution proofs, namely, that it works consistently until the proofs exceed \( 10^6 \) steps.

The remainder of the paper is structured as follows. Section~\ref{section:preliminaries} introduces the key concepts of propositional logic and the necessary context about the resolution proof system. We review the state-of-the-art approaches to proof minimization, as well as further context about the proof complexity, in Section~\ref{section:related-work}. The main novel contribution of this paper is introduced in Section~\ref{section:proofmin}, which describes the branch-and-bound procedure for enumerating the resolution proofs of an unsatisfiable formula and the optimizations used within it to improve the search time. We evaluate the performance of our approach, both for the time to optimality and for the proof length, in Section~\ref{section:experiments}. Finally, we conclude and discuss further research directions in Section~\ref{section:conclusions}. 

\section{Preliminaries} \label{section:preliminaries}

In this section, we establish the main concepts from propositional logic and SAT solving that we need to motivate the problem we address and the approach we propose. Section~\ref{subsection:propositional} contains the foundational definitions of propositional logic, such as the definition of an unsatisfiable formula. In Section~\ref{subsection:resolution}, we introduce the notion of the resolutional proof system and restate the most important facts involving it.

\subsection{Propositional Satisfiability}
\label{subsection:propositional}

We start by introducing propositional formulas in conjunctive normal form, which is a common expectation in modern SAT solvers. Given a Boolean variable \( x \), a \emph{literal} is either the variable \( x \) itself or its negation \( \bar{x}\). A \emph{clause} \( \omega \) is a (possibly empty) set of literals: \( \omega = \{ \ell_1, \dotsc, \ell_k \} \); a \emph{formula} \( F \) is a set of clauses: \( F = \{ \omega_1, \dotsc, \omega_m \}\). Since clauses and formulas are commonly interpreted in logical terms, we also use the conjunction/disjunction notation interchangeably with the set notation. For example, the terms \( \{ x, \bar{y} \} \) and \( x \vee \bar{y} \) correspond to the same clause, while the terms \( \{ \{ x, \bar{y} \}, \{ \bar{x}, z \} \} \) and \( (x \vee \bar{y}) \wedge (\bar{x} \vee z) \) correspond to the same formula.

Given a formula \( F \) over variables \( x_1, \dotsc, x_n \), a \emph{model} \( \boldsymbol{x} \) is a mapping from variables to Boolean values. A model \emph{satisfies} a clause \( \{ \ell_1, \dotsc, \ell_m \} \) if at least one of the literals evaluates to true, that is, there is an index \( j \) such that either \( \ell_j = x_t \) and \( \boldsymbol{x}_t = 1 \) or \( \ell_j = \bar{x}_t \) and \( \boldsymbol{x}_t = 0 \). By extension, a model satisfies a formula \( F \) if it satisfies all its clauses. If a model does not satisfy a clause or a formula, we say that the model \emph{falsifies} the clause or formula. Finally, we call a formula \emph{satisfiable} (SAT) if there is a model that satisfies it; otherwise, if all models falsify a formula, we call it \emph{unsatisfiable} (UNSAT).

\begin{note}
    We assume that clauses do not contain repeated variables, because \( x \vee x = x \) and \( x \vee \bar{x} = \top \). Thus, any repeated literals can be removed, and any clause with opposite literals can be removed from the formula without loss of generality: if a model satisfies the remaining clause, it also satisfies the entire formula.
\end{note}

\begin{note}
    An empty clause \( \bot = \emptyset \) is falsified by any model as there are no literals to evaluate to true; by extension, a formula \( F \) that contains an empty clause is falsified by any model.
\end{note}

Determining the SAT/UNSAT status of a given formula is an NP-complete problem \cite{Cook1971-Other,Levin1973-problInfTransm}, which also implies the following asymmetry between SAT and UNSAT statuses: if a formula \( F \) is reported to be SAT, this can be verified independently given a satisfying model, however, no similarly simple certificate could support the claim that \( F \) is UNSAT. There are, however, a few ways to do it; the next section introduces the certification approach studied in this paper.

\subsection{Resolution Proof System}
\label{subsection:resolution}

As stated in the introduction, the UNSAT claim is not trivial to verify. To address this, in this section, we introduce the resolution rule, which is a rule for deriving an implied clause from two given clauses (which themselves may have been derived earlier); in particular, the UNSAT claim in that context corresponds to deriving the empty clause.

\begin{definition}
    Given clauses \( A = C' \vee x \) and \( B = C'' \vee \bar{x} \), a \emph{resolution} rule is an inference rule that produces a clause \( C' \vee C'' \)---also denoted as \( A \diamond B \)---from \( A \) and \( B \): \[ \frac{C' \vee x, \quad C'' \vee \bar{x}}{C' \vee C''}. \] The resulting clause \( A\diamond B\) is called the \emph{resolvent}, and the variable \( x \) is referred to as \emph{pivot variable} (or pivot literal, which in this case points to the same entity).
\end{definition}

\begin{note}
    We can assume that \( A \) and \( B \) do not contain any other ``polar" literals other than \( x \) and \( \bar{x} \) because otherwise, the resolvent would contain those polar literals and therefore would be trivially true (and hence can be discarded right away). For example, \( (x\vee y) \diamond (\bar{x} \vee \bar{y}) = y \vee \bar{y} = \top \), where \( x \) is the pivot variable. In particular, if any two clauses can be resolved, there is no ambiguity about the pivot choice.
\end{note}

\begin{theorem}
    The resolution rule is \emph{sound} in the sense that any model satisfying clauses \( A \) and \( B \) also satisfies the resolvent clause \( A \diamond B \), and is \cite{Robinson1965-jAcm} in the sense that an empty clause can be derived from a formula if and only if it is UNSAT.
\end{theorem}

\begin{definition}
    \label{def:proof}
    Given an UNSAT formula \( F \), a (resolution) \emph{proof} (of unsatisfiability) is a sequence of \( M \) clauses \( (Q_1, \dotsc, Q_M) \) such that \( Q_M = \bot \) and for all \( 1 \le i < M \) either \( Q_i \in F \) (in which case \( Q_i \) is referred to as an \emph{axiom}) or \( Q_i = Q_j \diamond Q_k \) for some \( j, k < i \). The proof \emph{length} is the number of clauses \( M \) in it.
\end{definition}

\begin{example}
    \label{example:proof-definition}
    Consider an UNSAT formula \( F = \{ \{ x, \bar{y} \}, \{ \bar{x} \}, \{ y \} \} \). The following sequence of resolution steps derives an empty clause: \[ \underbrace{(\{ x, \bar{y} \} \diamond \{\bar{x} \})}_{=\{ \bar{y} \}} \diamond \{ y \}.\] 
    In our notation, this corresponds to a proof of length five with the following order of clauses: \( \{ x, \bar{y} \}, \{ \bar{x} \}, \{ \bar{y} \}, \{ y \}, \bot \). Note that the third and fourth clauses may be swapped without violating the proof definition.
\end{example}

Another way to encode the proof structure is to track the clauses derived in the proof and the mapping between the resolvents and their premises. This observation suggests the following equivalent definition of a proof:

\begin{definition}
    Given an UNSAT formula \( F \), a \emph{proof DAG} is a directed acyclic graph with the following structure: \begin{itemize}
        \item Vertices correspond to clauses over the variables of \( F \) and include the clauses of \( F \) and the empty clause.
        \item Any source vertex is a clause in \( F \), and any other vertex has two incoming edges.
        \item If \( \hat{\omega} \) has two incoming edges from \( \omega' \) and \( \omega'' \), then \( \hat{\omega} = \omega' \diamond \omega'' \).
    \end{itemize}
\end{definition}

\begin{proposition}
    Given a proof \( (Q_1, \dotsc, Q_M) \) of an UNSAT formula \( F \), there exists a proof DAG with vertices \( \{ Q_1, \dotsc, Q_M \}  \). Conversely, given a proof DAG of an UNSAT formula \( F \) with vertices \( V \), there exists a proof \( (Q_1, \dotsc, Q_M) \) such that \( \{ Q_1, \dotsc, Q_M \} = V \).
\end{proposition}

The resolution proof system is important for proof logging applications as it closely matches the inferences produced by SAT solvers. For example, the implementations of \emph{clause learning}, a step taken by modern SAT solvers that seek to add a new clause into the formula before backtracking, corresponds to revisiting the clauses that triggered the backtracking and taking their resolution in an appropriate order \shortcite{Beame2004-jArtifIntellRes}. \emph{Unit propagation}, a rule commonly used in SAT solvers that prescribes to assign literal \( \ell \) to be true if there is a clause \( C \vee \ell \) such that all literals in \( C \) are falsified, also corresponds to a sequence of resolutions of clauses that were used in the propagation:

\begin{example}
    Consider the use of unit propagation in Example~\ref{example:proof-definition}. Starting with an empty model, suppose that the unit propagation discovers the following decision sequence:
    \begin{enumerate}
        \item Clause \( \{ \bar{x} \} \) is unit, assign \( x \gets 0 \).
        \item Clause \( \{ x, \bar{y} \} \) is unit,\footnote{\( \bar{y} \) is the only unassigned literal; the rest of the clause has been falsified.} assign \( y \gets 0 \).
        \item Clause \( \{ y \} \) is unit; as assigning \( y \gets 1 \) is no longer possible, halt the procedure and declare the formula UNSAT.
    \end{enumerate}

    We can traverse this sequence of inferences in reverse, and that corresponds to the proof in Example~\ref{example:proof-definition}. Observe that the final conflict corresponds to the derivation \( \{y\} \diamond \{\bar{y}\} = \bot \). Going a step back, we recognize that \( \{\bar{y} \} \) was derived using \( \{ x, \bar{y} \} \). To complete the reasoning, we need to trace how the \( \bar{x} \) was assigned, which has happened in the first step via clause \( \{ \bar{x} \} \). We can summarize this by saying that \( \{\bar{y} \} = \{ x, \bar{y}\} \diamond \{ \bar{x} \} \), and, in turn, \( \{ y \} \diamond (\{ x, \bar{y}\} \diamond \{ \bar{x} \}) = \bot \) -- which is exactly the proof introduced earlier.
\end{example}

Given an UNSAT formula, there are usually many different valid proofs of that fact. The discussion above indicates that the shorter proofs correspond to solver executions making fewer steps and are thus preferable, as they correspond, under simplifying assumptions, to faster executions of a SAT solver. We reflect this observation in the following definition, which introduces the problem we address in the remainder of this paper.

\begin{definition}
Let \( F \) be an UNSAT formula, and \( \mathtt{Proofs}[F] \) be the set of all its proofs in the sense of Definition~\ref{def:proof}. Then the \emph{proof minimization} problem is the optimization problem of form \( \min\left\{ \# R \mid R \in \mathtt{Proofs}[F] \right\} \).\footnote{Throughout the text, we use the notation \( \# X \) for the cardinality of a set \( X \).}
\end{definition}

\section{Related Work} \label{section:related-work}

One of the simplest observations about the definition of proof is that it may contain irrelevant inferences that do not impact the correctness of the proof. This observation is translated into the \emph{proof trimming} techniques \shortcite{Heule2013-Other}, a strategy for proof compression that begins with the proof conclusion (i.e., empty clause), marks inferences directly used to generate the conclusion, then the inferences used to derive \emph{those} steps, etc., and retaining only the marked inferences. This idea also appears in other proof systems under different aliases, such as tightening in the VIPR certification approach for mixed-integer programming \shortcite{Cheung2017-Other}. In application to SAT solvers, DRAT-trim \shortcite{Wetzler2014-Other} is the state-of-the-art tool for trimming clausal proofs.

The next logical step from this idea is to consider local transformations of proofs and apply a local search technique to produce shorter proofs. For the branch-and-bound trees in mixed-integer programming framework, this has been done by \shortciteA{Munoz2023-Other} by going through the subtrees and attempting to trim that subtree or discover a branching direction that derives the dual bound as good as the subtree bound. A similar idea for SAT solving is represented by the RANGER approach from \citeA{Prestwich2006-Other}, which operates on clause sets of fixed size and repeatedly deletes clauses, adds the resolvents of available clauses, or re-introduces the formula clauses, and does so until an empty clause is introduced. The distinction here is that RANGER does not attempt to compress the proof, rather, it seeks to establish a local search approach that could produce an UNSAT verdict.

Last but not least, one may treat the problem of discovering short proofs as an optimization problem, with the feasible set being defined by the proof system inference rules and the objective function encoding the proof length, typically the number of inference steps. For propositional solving, the seminal work of \citeA{Peitl2021-jArtifIntellRes} introduces this idea by constructing a propositional formula that is true if and only if a given formula \( F \) has a proof of length \( s \), and improving the search performance by enforcing the topological sort on the underlying proof DAG in order to break underlying symmetries. However, as we observe in our work, this strategy does not break all symmetries; in particular, this strategy does not account for different derivations of the same clause sets, which we exploit in Section~\ref{subsection:resolution}.

These ideas are used for the proof minimization by invoking a SAT solver with increasing values of \( s \) until the formula becomes satisfiable, in which case a satisfying model contains the proof clauses. Starting the search with \( s = 1 \) is possible, but this work also justifies the choice of a higher starting value for the proof length, which corresponds to the lower bound of the length of a valid proof. Additionally, this work also establishes the \emph{resolution number} \( h_m \) as the largest smallest proof length among the formulas with \( m \) clauses, and describes a subset of such formulas that is sufficient to enumerate in order to derive the values of \( h_m \).

Similar ideas have recently been introduced for tree-like proof systems. For example, \shortciteA{Dey2024-mathProgram} formulate the problem of finding the smallest branch-and-bound tree for 0---1 integer programs as a dynamic program, and compare the optimal tree sizes for various classes of problems against the strong branching variable selection. Another variation on this idea can be found in \shortciteA{Sidorov2024-Other}, developing an approach for constructing compact tree proofs of optimality in constrained shortest path problems; the distinctive trait of that work is that it focuses on reasoning in an \emph{ad-hoc} proof system designed for interpretability, as opposed to a proof system encoding the solver inference steps.

\section{Proof Minimization} \label{section:proofmin}

This section introduces our methodology for discovering the shortest proofs of UNSAT formulas; we structure our narrative as follows: \begin{itemize}
    \item Section~\ref{subsection:representation} introduces our novel \emph{layer list representation} and shows via Theorem~\ref{thm:uniqueness} that this representation breaks all permutation symmetries. More specifically, we demonstrate that any correct proof of the unsatisfiability of a formula maps to exactly one such representation. We also motivate this design of the representation by showing that it improves on the symmetry breaking on DAGs by \shortciteA{Fichte2020-Other}, as there are \emph{non-equivalent} proof DAGs that have the same encodings in our approach.
    \item Section~\ref{subsection:algorithm} gives a high-level overview of the \emph{branch-and-bound} approach to discovering shortest proofs. The distinctive trait of the search procedure we put forward is that it operates on layer list representations and thus inherits its symmetry-breaking properties. This part also lists the subroutines the algorithm relies upon and the properties it assumes to hold on them.
    \item Section~\ref{subsection:partition} introduces the \emph{branching} scheme that complies with the subproblem representation without introducing the full enumeration of subsets of a given formula.
    \item Section~\ref{subsection:frontier} shows that the subproblems can be simplified during the search by discarding subsumed clauses, an idea which we capture in the \emph{frontier set} definition. We also show that this simplification can be used as a pruning strategy as soon as there is a clause that is neither a premise of some resolution step nor a frontier clause.
    \item Section~\ref{subsection:dominance} establishes the \emph{dominance} relation on the subproblems. The defining trait of this relation is that if one subproblem dominates another subproblem, then any proof that can be recovered from the latter subproblem corresponds to some proof that (a) can be recovered from the former subproblem and (b) is not longer than the original proof. In terms of the optimization algorithm, it means that dominated subproblems can be ignored without the risk of discarding the optimal solution.
    \item Section~\ref{subsection:lower-bound} completes the description of our approach by establishing the \emph{lower bound} on the proof length via the cardinality of the smallest minimal unsatisfiable subset. We also discuss a procedure that is used to establish a bound on that quantity, which by transitivity is also a valid bound for our problem.
\end{itemize}

\subsection{Layer List Representation} \label{subsection:representation}

The proof minimization problem is challenging for, among other reasons, its highly symmetric nature: for example, reordering the clauses without putting premises after the conclusions yields different proofs. Since the previous work has represented the proofs as models of a propositional formula, the established way to break symmetries is to (i) encode the solutions as proof DAGs and (ii) impose the symmetry-breaking constraints of the form ``This sequence is a topological ordering of a DAG with lexicographic ordering on disconnected clauses" \cite{Fichte2020-Other}. However, we argue that breaking only those symmetries is insufficient because the proofs with the same sets of derived clauses \emph{can be encoded by different DAGs}; the next example illustrates this issue.

\begin{example}
    \label{example:more-symmetries}
    Consider an UNSAT formula \[ F = \bigg\{ \big\{ y, \bar{t} \big\}, \{ x, y, z, t \}, \{ \bar{x}, y \}, \{ \bar{y}, z \}, \{ \bar{z}, t \}, \big\{ \bar{x}, \bar{t} \big\}, \{ x, \bar{y} \} \bigg\} \] and a proof encoded by the sequence starting with the seven axioms and deriving clauses \[ \{ \bar{x}, z \}, \{ y, z, t \}, \{ \bar{x}, t \}, \{ y, t \}, \{ y \}, \{ \bar{x} \}, \{ \bar{y} \}, \bot. \] This proof can be encoded by \emph{two distinct} DAGs using only the derived clauses represented in Figure~\ref{fig:more-symmetries-dag}, with the difference being that the clause \( \{ y, z, t \} \) is derived as \( \{ x, y, z, t \} \diamond \{ \bar{x}, y \} \) in the left DAG and as \( \{ x, y, z, t \} \diamond \{ \bar{x}, z \} \) in the right DAG.

    \begin{figure}[!ht]
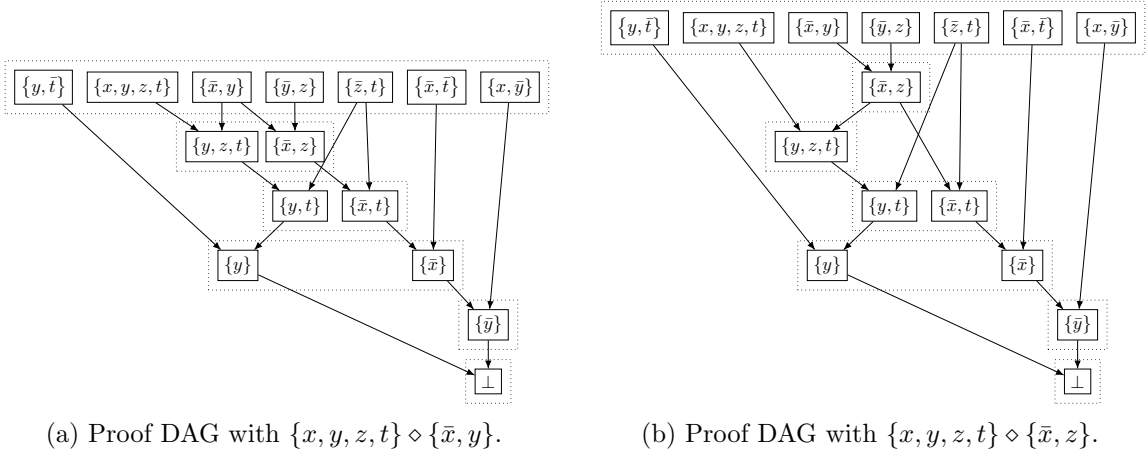

    \centering
    \subcaptionbox{Proof DAG with \( \{ x, y, z, t \} \diamond \{ \bar{x}, y \} \).\protect\label{fig:more-symmetries-dag-top}}[0.48\textwidth]{
        \resizebox{0.48\textwidth}{!}{
        \includegraphics{dag-top.tikz}}
    }
    \hfill
    \subcaptionbox{Proof DAG with \( \{ x, y, z, t \} \diamond \{ \bar{x}, z \} \).\protect\label{fig:more-symmetries-dag-bot}}[0.48\textwidth]{
        \resizebox{0.48\textwidth}{!}{
        \includegraphics{dag-bot.tikz}}
    }
    \caption{Two DAG encodings of the proof from Example~\ref{example:more-symmetries} using the same vertices but different edges. Every row on both figures corresponds to one layer.}
    \label{fig:more-symmetries-dag}
    \end{figure}

    Given that, consider the symmetry-breaking constraint that restricts the search space to the topological orderings of proof DAGs and breaks ties with lexicographic ordering given by \( x < \bar{x} < y < \bar{y} < z < \bar{z} < t < \bar{t} \). That constraint would admit the following two sequences\footnote{For the sake of brevity, we introduce only the prefixes of the canonical topological order with the first five clauses. For the purposes of our narrative, the suffixes are not important, because no choice of suffixes can make the two listed sequences equal.} of clauses despite only differing in their order, because they correspond to the topological ordering of different graphs: \begin{itemize}
        \item \( \{ x, y, z, t \}, \{ x, \bar{y} \}, \{ \bar{x}, y \}, \{ \bar{x}, t \}, \boldsymbol{\{ y, z, t \}}, \dotsc  \) for the left DAG, and
        \item \( \{ x, y, z, t \}, \{ x, \bar{y} \}, \{ \bar{x}, y \}, \{ \bar{x}, t \}, \boldsymbol{\{ \bar{y}, z \}}, \dotsc \) for the right DAG.
    \end{itemize}
\end{example}

Example~\ref{example:more-symmetries} not only shows that the DAG symmetry breaking misses some opportunities to reduce the search space; it also illustrates a concrete flaw in such a symmetry-breaking strategy, namely, it discriminates proofs having the same clauses but deriving them differently. Returning to the proofs introduced in the example, the clause \( \{ y, z, t \} \) is \emph{directly} implied by axioms in Figure~\ref{fig:more-symmetries-dag-top}, whereas Figure~\ref{fig:more-symmetries-dag-bot} prefaces this proof step with deriving \( \{ \bar{x}, z \} \), a clause used later to derive \( \{ y, z,  t \} \). However, this would not be a problem had we additionally fixed the level of indirection of a clause, for example, by constraining the DAGs to derive any clause used in the proof as soon as possible; that, in particular, would invalidate Figure~\ref{fig:more-symmetries-dag-bot} but retain Figure~\ref{fig:more-symmetries-dag-top}.

With that in mind, we introduce the structure that addresses this concern on top of symmetries broken by the DAG representation.

\begin{definition}
    \label{def:layer-list}
    Given an UNSAT formula \( F \), the \emph{layer list} \( (L^0, L^1, \dotsc, L^n) \) on \( F \) is the sequence of sets of clauses (referred to as \emph{layers}) satisfying the following properties: \begin{enumerate}
        \item \emph{Initialization}. The first layer \( L^0 = F \) includes all the axioms and only them.
        \item \emph{Termination}. One of the layers contains the empty clause: \( \exists k : 1 \le k \le n, \bot \in L^k \).
        \item \emph{Consistency}. All clauses in every layer, except the first, are obtained by resolving a clause from the preceding layer with a clause from any earlier layer: \[ \forall 1 \le j \le n, \omega \in L^j \implies \omega \in \left\{  \omega' \diamond \omega'' \mid \omega' \in L^{j-1}, \omega'' \in L^0 \cup \cdots \cup L^{j-1} \right\}. \]
        \item \emph{Take-it-or-leave-it property}. Introducing any clause to an earlier layer violates the consistency property: \[ \forall 1 \le k < j \le n, \omega \in L^j \implies \omega \notin \left\{ \omega' \diamond \omega'' \mid \omega' \in L^{k-1}, \omega'' \in L^0 \cup \cdots \cup L^{k-1} \right\}. \]
    \end{enumerate}
\end{definition}

\begin{note}
    From the proof minimization point of view, it makes sense to restrict the termination condition to enforce the \emph{last} element of the layer list to contain the empty clause and no other clauses: \( L^n = \{ \bot \} \). However, we do not enforce this in our definition, because then it would not map naturally to proofs with redundant clauses, which would make the statement of Theorem~\ref{thm:uniqueness} less straightforward. On the other hand, we enforce this restricted version of the layer list definition in our computations by halting the search once the layer with the empty clause has been produced.
\end{note}

The first three conditions replicate the proof DAG definition, as they assert that the layers constitute a topological ordering of a proof DAG, with vertices within a layer having no restrictions on their pairwise placement. In particular, any layer list can be mapped to a proof by writing out the first layer clauses in any order, then writing out the second layer clauses, and so on until the final layer, which concludes the proof. The take-it-or-leave-it property addresses the concern from Example~\ref{example:more-symmetries} by mandating that any clause used in the proof is used in the earliest available layer, as illustrated by the following example.

\begin{example}
    To show how our definition aids in breaking symmetries, consider again two proofs from Figure~\ref{fig:more-symmetries-dag}. Figure~\ref{fig:more-symmetries-dag-top} can be encoded as a layer list by grouping the clauses in boxes bounded by dotted lines and arranging layers from top to bottom. The resulting layer list looks as follows:

    \begin{align*}
        L^0 &= \Big\{ \big\{ y, \bar{t} \big\}, \{ x, y, z, t \}, \{ \bar{x}, y \}, \{ \bar{y}, z \}, \{ \bar{z}, t \}, \big\{ \bar{x}, \bar{t} \big\}, \{ x, \bar{y} \} \Big\} \\
        L^1 &= \{ \{ \bar{x}, z \}, \{ y, z, t \} \} \\
        L^2 &= \{ \{ \bar{x}, t \}, \{ y, t \} \} \\
        L^3 &= \{ \{ y \}, \{ \bar{x} \} \} \\
        L^4 &= \{ \{ \bar{y} \} \} \\
        L^5 &= \{ \bot \}.
    \end{align*}

    The first three properties of a layer list can be validated by following the arrows in Figure~\ref{fig:more-symmetries-dag-top} and validating that they indeed lead from premises to resolvents. The take-it-or-leave-it property also holds but requires more effort to verify.
    
    For the third and fourth layers, it can be validated by recognizing that a unit clause \( \{ \ell \} \) can be derived either as \( \{ \ell, v \} \diamond \{ \ell, \bar{v} \} \) or by \( \{ \ell, v \} \diamond \{ \bar{v} \} \) and acknowledging that neither is possible without using the preceding layer. For example, the only clauses from \( L^0 \cup L^1 \) that contain \( \bar{x} \) are \( \{ \bar{x}, y \} \), \( \{ \bar{x}, z \} \), and \(  \big \{ \bar{x}, \bar{t} \big \} \); since no two clauses can be resolved, the clause \( \{ \bar{x} \} \) cannot appear on layer 1 or earlier.
    
    For the second layer, we can observe that \( \{ y, t \} \) has to have a premise with literal \( t \); the only clauses fitting this description in \( L^0 \) are \( \{ x, y, z, t \} \), which cannot be a premise of \( \{ y, t \} \), and \( \{ \bar{z}, t \} \), which would require a non-axiom \( \{ y, z \} \) as the other premise. Therefore, \( \{ y, t \} \) cannot be obtained purely from axioms, satisfying the take-it-or-leave-it property.

    On the other hand, the properties of the layer list do not hold for Figure~\ref{fig:more-symmetries-dag-bot}, as \( \{ y, z, t \} \in L^2 \) violates the take-it-or-leave-it property. This clause can be represented as the resolvent of two clauses from \( L^0 \), which precludes it from being placed in \( L^2 \) or any subsequent layer.
\end{example}

We now formalize and prove the key fact of this subsection which establishes that for any proof there is a ``canonical" layer list representation that depends on the used clauses \emph{but not on their order in the proof}.

\begin{theorem}[Layer list uniqueness]
    \label{thm:uniqueness}
    Consider a formula \( F \) and a set of clauses \( P, P \cap F = \emptyset \) such that \( F \cup P \) can be ordered into a proof of \( F \). Then there is exactly one layer list \( (L^0, L^1, \dotsc, L^n) \) that derives the clauses from \( P \) and only them: \( L^1 \cup \cdots \cup L^n = P. \)
\end{theorem}

\begin{proof}
    First, observe that if \( (L^0, \dotsc, L^n) \) and \( (M^0, \dotsc, M^{n'}) \) are layer lists that derive the same set of clauses, then \( n = n' \) and \( L^k = M^k \) for all \( k \). If that is not the case, we consider the earliest disagreeing layer \( L^k \neq M^k \) with \( L^1 = M^1, \dotsc, L^{k-1} = M^{k-1} \). Assume without loss of generality that there is clause \( \omega \) such that \( \omega \in L^k \) but \( \omega \notin M^k \). Given that both proofs derive the same clauses, \( \omega \) can be found in a later layer of \( M \) even though it can be derived from the layers \( M^0, \dotsc, M^{k-1} \), violating the take-it-or-leave-it property for \( M \).

    Now consider the following construction that places the clauses in the next layer by filtering out used clauses and retaining only the resolvents of a clause in the preceding layer with a clause on one of the earlier layers, terminating when the clauses in \( P \) are exhausted: \begin{align*}
        L^0 &:= F, \\
        L^1 &:= P \cap \{ \omega' \diamond \omega'' \mid \omega', \omega'' \in L^0 \}, \\
        L^2 &:= (P \setminus L^1) \cap \{ \omega' \diamond \omega'' \mid \omega' \in L^0, \omega'' \in L^0 \cup L^1 \}, \\
        &\cdots \\
        L^n &:= (P \setminus (L^1 \cup \cdots \cup L^{n-1})) \cap \{ \omega' \diamond \omega'' \mid \omega' \in L^{n-1}, \omega'' \in L^0 \cup \cdots \cup L^{n-1} \}, \\
        P &= L^1 \cup \cdots \cup L^{n-1} \cup L^n.
    \end{align*}
    If this procedure terminates, its output would immediately satisfy all the points in Definition~\ref{def:layer-list}: \( L^0 = F \), \( \bot \in P \) implies that \( \bot \in L^k \) for some positive \( k \), the consistency property is ensured by only considering the resolvent clauses in the definitions of \( L^j \), and take-it-or-leave-it property is enforced by removing the clauses from already constructed layers.

    To complete the proof, we show that all the layers are non-empty: this would imply that the described procedure terminates, because \( P \) is finite and the \( L^k \) sets are pairwise disjoint, so the cardinality of the set \( P \setminus (L^1 \cup \cdots \cup L^{n-1}) \) monotonically decreases and thus eventually has to reach zero. For a proof by contradiction, suppose that there is an empty layer, and let \( L^k \) be the first such empty set; also, let \( P' \) be the set of clauses that have not found their way to the first \( (k-1) \) layers: \[ P' := P \setminus (L^1 \cup \cdots \cup L^{k-1}). \] Since the procedure has \emph{not} terminated, \( P' \) is not empty. Given that, we show by induction on the cardinality of \( P' \) that one of the clauses \( \omega \in P' \) can be expressed as \( \omega = \omega' \diamond \omega'' \), where both premises are either axioms or can be discovered in earlier layers. This statement would be a contradiction with the assumption that \( L^k \) is empty, as the clause \( \omega \) is in \( P \), not in one of the previous layers, and can be represented as a resolvent of two previous clauses.
    
    For a proof by induction, consider first the base case is \( P' = \{ \omega \} \) being a singleton set; since \( P' \) is a subset of a proof excluding the axioms, the desired clauses can be recovered from the original proof. For the induction step, if \( \omega \in P' \) and \( \omega \) cannot be represented the same way as in the base case, let \( P' \) be the subset of clauses without \( \omega \) and all clauses from \( P' \) that depend on it. To invoke the induction hypothesis, we need to show that \( P'' \) is not empty. This is the case because otherwise \emph{all} clauses of \( P' \) save for \( \omega \) use \( \omega \) as one of its premises, which means that \( \omega \) cannot be derived from previous layers (as it would be placed in one of the previous layers) nor from \( P' \) (because all of those clauses depend on it).
\end{proof}

To summarize, in this section, we have introduced a representation of proofs that does not reproduce the same set of clauses more than once but is much simpler to generate procedurally. The construction from Theorem~\ref{thm:uniqueness} suggests that enumerating all proofs is the same as enumerating all non-empty subsets of \( \{ \omega' \diamond \omega'' : \omega', \omega'' \in L^0 \}  \), then enumerating for every subset \( L^1 \) enumerate all subsets of \( \{ \omega' \diamond \omega'' : \omega' \in L^0, \omega'' \in L^0 \cup L^1 \} \) filtering out clauses that could be derived from \( L^0 \), and proceed to do so until an empty clause is discovered. In the next subsection, we develop this insight by introducing a branch-and-bound algorithm operating on layer list prefixes.

\subsection{Branch-and-bound Search, High-level Description} \label{subsection:algorithm}

As observed earlier, we can see the enumeration of the proofs as a recursive procedure on layer list prefixes starting from \( (L^0 = F) \) and extending any given prefix \( (L^0, \dotsc, L^k) \) with the next layer \( L^{k+1} \) by: \begin{itemize}
    \item constructing the superset \( \Lambda' := \{ \omega' \diamond \omega'' : \omega' \in L^k, \omega'' \in L^0 \cup L^1 \cup \cdots \cup L^k \} \),
    \item filtering out clauses \( \Delta^j := \{ \omega' \diamond \omega'' : \omega', \omega'' \in L^0 \cup L^1 \cup \cdots \cup L^j \} \) for all previous layers \( j < k \) to restrict the clause set to \( \Lambda := \Lambda' \setminus \Delta^{k-1} \),
    \item and making a recursive call on prefixes \( (L^0, \dotsc, L^k, L^{k+1}) \) for all non-empty \( L^{k+1} \subseteq \Lambda \).
\end{itemize}
From this point, we view the prefixes of layer lists as \emph{subproblems} of the original proof minimization problem in the following sense. In the original problem, we look for the layer list that starts with \( L^0 = F \) and has the fewest clauses across all the layers. Next, we generate the prefixes \( (L^0, L^1) \) and for each of them we look for the layer list that starts with those \( (L^0, L^1) \) and has the fewest clauses; the best of those layer lists is going to be the one with fewest clauses in the original problem. With each of \emph{those} prefixes, we do the same procedure looking for the smallest layer lists with larger fixed prefixes and continue doing so until we exhaust all possible layer lists.

We also note that this procedure does not require tracking \emph{each} layer separately. More specifically, it is sufficient to have the clauses from the previous layer \( L^k \), the clauses from the union of all previous layers \( \hat{L}^k := L^0 \cup L^1 \cup \cdots \cup L^k \), and the ``filtering" sets \( \Delta^{k-1} \). Given that, we can collate all the layers, except the last, into a single set in our previous description of the search procedure. The next definition reflects these observations in a data structure used to implement the outlined search procedure.

\begin{definition}
    \label{def:subproblem}
    Given a propositional formula \( F \), a \emph{subproblem} \( \mathcal{P} \) is a tuple containing: \begin{itemize}
        \item a set of \emph{previous layer clauses} \( \Previous{\mathcal{P}} \),
        \item a set of \emph{current layer clauses} \( \Current{\mathcal{P}} \),
        \item a set of \emph{next layer clauses} \( \Next{\mathcal{P}} \),
        \item and a set of \emph{forgotten clauses} \( \Forgotten{\mathcal{P}} \).
    \end{itemize}
    The first three clause sets are collectively labelled as \emph{known} clauses: \[ \Known{\mathcal{P}} = \Previous{\mathcal{P}} \cup \Current{\mathcal{P}} \cup \Next{\mathcal{P}}. \]
    The \emph{root subproblem} \( \mathcal{P}_{root} \) is defined as follows: \(  \Previous{\mathcal{P}_{root}} = \Current{\mathcal{P}_{root}} = F \) and \( \Next{\mathcal{P}_{root}} = \Forgotten{\mathcal{P}_{root}} = \emptyset \).
\end{definition}

The individual parts of this definition are interpreted as follows. If the formula layer and the first \( k \) layers of the layer list are fixed, then \( \Previous{\cdot} \) is the union of all constructed layers \( \hat{L}^{k-1} \), excluding the most recent one, whereas \( \Current{\cdot} \) is the most recent complete layer \( L^k \). \( \Next{\cdot} \) is a \emph{subset} of the layer \( L^{k+1} \) that follows the fixed part of the layer list; we define this part of the subproblem like this to add/discard clauses of the next layer one by one, as opposed to enumerating all possible subsets. Thus, the clauses in \( \Known{\cdot} \) are all the clauses that have been added to the proof in this layer list prefix.

Finally, \( \Forgotten{\cdot} \) is a superset of the clauses \( \hat{\Delta}^k \) not used on earlier occasions that also contains some of the clauses in \( \Delta^{k+1} \) removed from the layer \(L^{k+1}\) that is currently under construction. Unlike the other parts of the subproblem definition, \( \Forgotten{\cdot} \) does not translate to any part of a resulting layer list. Instead, the purpose of the  \( \Forgotten{\cdot} \) is to enforce the take-it-or-leave-it property by (i) storing clauses satisfying the consistency property for preceding layers while being absent from the proof and (ii) disregarding them during the construction of the following layers.

\begin{figure}[t!]
\centering
\includegraphics{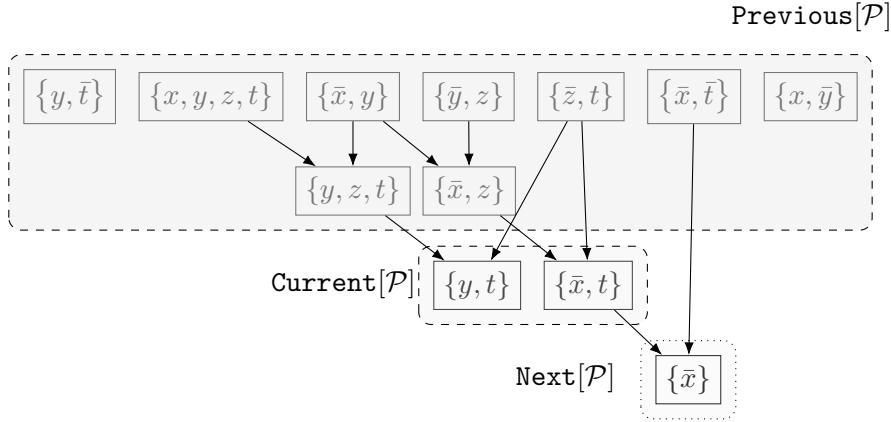}
\caption{A subproblem compatible with the proof in Figure~\ref{fig:more-symmetries-dag-top}; the second-to-last layer of the proof is a superset of the next set \( \Next{\mathcal{P}} \) of the subproblem.}
\label{fig:subproblem}
\end{figure}

To illustrate the individual components of this definition, Figure~\ref{fig:subproblem} encodes a subproblem corresponding to a proof prefix from Figure~\ref{fig:more-symmetries-dag-top}. The following example elaborates on the interactions between the subproblem components. 

\begin{example}
    \label{example:subproblem}
    In the subproblem from Figure~\ref{fig:subproblem}, the definition components are as follows: \begin{itemize}
        \item \( \Previous{\mathcal{P}} = F \cup \{ \{y, z, t \}, \{ \bar{x}, z \} \} \) stores the union of the formula and \( L^1 \),
        \item \( \Current{\mathcal{P}} = \{ \{y, t \}, \{ \bar{x}, t \} \} \) stores \( L^2 \) as the last full layer in the prefix,
        \item \( \Next{\mathcal{P}} = \{ \{ \bar{x} \} \} \) stores the \( \{ \bar{x} \} \) clause from \( L^3 \) (encoding the decision to include it in the proof) but not the \( \{ y \} \) clause. Note that this does not imply any decision made on that clause.
        \item \( \Forgotten{\mathcal{P}} \supseteq \{ \{ z, t \}, \{ y, z \}, \{ \bar{y}, t \} \} \); we do not write out the full set for the sake of brevity, but the three listed clauses are in the forgotten set, as they can be decomposed into a resolution of two earlier clauses: \begin{align*}
            \{ z, t \} &= \underbrace{\{ y, z, t \}}_{\in \Previous{\mathcal{P}}} \diamond \underbrace{\{ \bar{y}, z \}}_{\in \Previous{\mathcal{P}}} \\
            \{ y, z \} &= \underbrace{\{ y, z, t \}}_{\in \Previous{\mathcal{P}}} \diamond \underbrace{\{ y, \bar{t} \}}_{\in \Previous{\mathcal{P}}} \\
            \{ \bar{y}, t \} &= \underbrace{\{\bar{y}, z \}}_{\in \Previous{\mathcal{P}}} \diamond \underbrace{\{ \bar{z}, t \}}_{\in \Previous{\mathcal{P}}}. \\
        \end{align*} Following up on the previous point, \( \{ y \} \in \Forgotten{\mathcal{P}} \) encodes the decision to \emph{not} include this clause in the proof, whereas \( \{ y \} \notin \Forgotten{\mathcal{P}} \) in this context means that this clause is not yet decided to be in the proof.
    \end{itemize}
\end{example}

To navigate between the notions of a subproblem and a layer list, we also introduce the definition encoding the concept of a ``feasible solution" of a subproblem.

\begin{definition}
    Given a propositional formula \( F \) and a subproblem \( \mathcal{P} \), a \emph{compatible proof} is a layer list such that one of the layers \( L^k \) is equal to \( \Current{\mathcal{P}} \), the union of all preceding layers is equal to \( \Previous{\mathcal{P}} \), the next layer is a superset of \( \Next{\mathcal{P}} \), and no clauses from \( \Forgotten{\mathcal{P}} \) occur in the layer list.
\end{definition}

Algorithm~\ref{alg:branch-bound-and-commit} formalizes the insights introduced earlier in a single search procedure. The algorithm works in a branch-and-bound fashion by maintaining a queue of unexplored subproblems and an incumbent shortest proof. Each iteration takes an unexplored subproblem off the queue, prunes it if possible, updates the incumbent if needed, and splits it into new subproblems. In particular, as it normally happens with branch-and-bound algorithms, our approach is an \emph{anytime} approach, meaning that stopping it at any point after the root subproblem has been processed results in a valid resolution proof from the incumbent \( L^* \) and a valid lower bound on the proof length.

As common in branch-and-bound procedures, we need some \( \mathtt{Partition} \) procedure to partition a subproblem into smaller subproblems without discarding the optimal layer list and a \( \mathtt{Bound} \) procedure to provide a valid lower bound for a proof length of a compatible proof of a subproblem. Additionally, we use a \( \mathtt{Complete} \) procedure that provides a valid but not necessarily shortest proof starting from the clauses of the current subproblem, and the \( \mathtt{Lookup} \) procedure that checks if a previously visited subproblem dominates the current subproblem. In the next proposition, we specify the conditions on these procedures sufficient for the whole algorithm to be correct; the next subsections explain each of these subroutines in detail.

\begin{algorithm}[ht]
\caption{Branch-and-bound algorithm for resolution proof minimization}\label{alg:branch-bound-and-commit}
\KwData{Unsatisfiable propositional formula \( F \).}
\KwResult{The shortest resolution proof \( L^* \) of unsatisfiability of \( F \).}
\tcp{Subproblem queue; initialized with the root subproblem}
\( \mathcal{Q} \gets \{ \mathtt{RootSubproblem}(F) \} \)\;
\tcp{Incumbent proof}
\( L^* \gets \Complete{F} \)\;
\tcp{Dominance cache}
\( \mathcal{D} \gets \emptyset \)\;
\While{\(\mathcal{Q} \neq \emptyset\)}{
    \tcp{Take the next subproblem from the queue}
    \( \mathcal{P} \gets \mathtt{Pop}(\mathcal{Q}) \)\;
    \tcp{Look up the dominance cache, prune if a dominating node has already been explored}
    \If{\(\Lookup{\mathcal{P}}{\mathcal{D}}\)}{
        \Continue\;
    }
    \tcp{Bound the remaining proof length, prune if no valid proof can improve the incumbent}
    \If{\(\Bound{\mathcal{P}} \ge \# L^*\)}{
        \Continue\;
    }
    \tcp{Store the subproblem in the dominance cache}
    \( \mathcal{D} \gets \mathcal{D} \cup \{ \mathcal{P} \} \)\;
    \tcp{Partition the current subproblem}
    \For{\( \mathcal{P}_{sub} \in \Partition{\mathcal{P}} \)}{
        \( \mathcal{Q} \gets \mathcal{Q} \cup \{ \mathcal{P}_{sub} \} \)\;
        \tcp{Complete the new subproblems to full proofs, update the incumbent if possible}
        \( L_{sub} \gets \Complete{\Known{\mathcal{P}_{sub}}} \)\;
        \If{\( \# L_{sub} < \# L^* \)}{
            \( L^* \gets L_{sub} \)\;
        }
    }
}
\Return \( L^* \)\;
\end{algorithm}

\begin{proposition} \label{proposition:correctness}
    Algorithm~\ref{alg:branch-bound-and-commit} terminates and returns the shortest proof of an unsatisfiable propositional formula if the following assumptions hold for any subproblem \( \mathcal{P} \): \begin{enumerate}
        \item For any compatible proof of \( \mathcal{P} \) there is a compatible proof of \( \mathcal{P}' \) of at most the same length for some \( \mathcal{P}' \in \Partition{\mathcal{P}} \).
        \item Any compatible proof of \( \mathcal{P} \) either has length at least \( \Bound{\mathcal{P}} \) or can be transformed (without increasing length) to a compatible proof of another subproblem.
        \item \( \Lookup{\mathcal{P}}{\mathcal{D}} \) implies that there is a subproblem \( \mathcal{P}' \in \mathcal{D} \) such that any compatible proof of \( P \) can be transformed into a compatible proof of \( \mathcal{P}' \) without increasing its length.
        \item \( \Complete{\mathcal{P}} \) returns a valid resolution proof of \( \Known{\mathcal{P}} \).
    \end{enumerate}
\end{proposition}

Before proceeding to a more specific discussion, we first explain the design decisions of the more straightforward parts of this algorithm. We implement the queue \( \mathcal{Q} \) as the priority queue that pops the subproblem with the most optimistic lower bound, thereby making this approach the best-first search strategy; ties are broken in favor of more recent subproblems. The \( \mathtt{Complete} \) function triggers CaDiCaL on the \( \Known{\mathcal{P}} \) as an input formula, extracts its proof in LRAT format as introduced by \shortciteA{Cruz-Filipe2017-Other}, and decodes each derived clause into propagation sequence by traversing the dependent clauses back-to-front. The remainder of Section~\ref{section:proofmin} finishes the algorithm description by establishing the subroutines compliant with Proposition~\ref{proposition:correctness}.

\subsection{Subproblem Partitioning} \label{subsection:partition}

Given that our search space is represented as a list of sets, it is natural to assume that the branching rule is implemented by enumerating all valid sets for the next list element. However, as pointed out earlier, this amounts to enumerating all subsets of the resolvent clauses without forgotten clauses. This is problematic because the set of resolvent clauses can have a quadratic size with respect to the formula size. Fortunately, the subproblem representation explicitly lists the disposed clauses using the forgotten set; in this subsection, we leverage it to design a branching strategy that generates at most one subproblem per resolvent of a clause in \( \Current{\cdot} \) and \( \Current{\cdot} \cup \Previous{\cdot} \).

\begin{example}
    \label{example:branching}
    Consider the subproblem from Example~\ref{example:subproblem} and observe that we can list the clauses that could be in \( \Next{\mathcal{P}'} \) by evaluating pairwise resolutions between the current clauses and the current/previous clauses. In our case, this gives\footnote{Excluding the clauses already in the proof.} us a candidate list \[ \{ y \}, \{ z, t \}, \{ x, t \}, \{ \bar{y}, t \}. \]

    However, the previous example also establishes that the clauses \( \{ z, t \} \) and \( \{ \bar{y}, t \} \) are in the forgotten set and hence must be deleted. The interpretation we can give here is that those clauses could have been derived on earlier layers and therefore were derived in other parts of the search tree, making the decisions about them in the current subproblem pointless. We are left with clauses \( \{ y \}, \{ x, t \} \); suppose now that this is an \emph{ordered} list.

    All compatible proofs of \( \mathcal{P} \) can be split into three pairwise disjoint groups: \begin{itemize}
        \item \( \{ y \} \in \Next{\mathcal{P}}\). This corresponds to deriving \( \{ y \}\) in the layer \( L^3 \) and postponing the decision about the \( \{ x, t \} \) clause. These decisions can be encoded by adding \( \{ y \} \) to the next set, as illustrated in Figure~\ref{fig:branching-y}.
        \item \( \{ y \} \notin \Next{\mathcal{P}}, \{ x, t \} \in \Next{\mathcal{P}} \). This corresponds to adding \( \{ x, t \}\) to the layer \( L^3 \) and never touching \( \{ y \} \). These decisions can be encoded by adding \( \{ y \} \) to the forgotten set and  \( \{ x, t \} \) to the next set, as illustrated in Figure~\ref{fig:branching-xt}.
        \item \( \{ y \}, \{ x, t \} \notin \Next{\mathcal{P}} \). This corresponds to deriving only \( \{ \bar{x} \}\) in the layer \( L^3 \) and never touching \( \{ y \}, \{ x, t \} \). These decisions can be encoded by adding \( \{ y \}, \{ x, t \} \) to the forgotten set and shifting previous/current/next sets as illustrated in Figure~\ref{fig:branching-commit}.
    \end{itemize}

    \begin{figure}
    \centering
    \subcaptionbox{Subproblem that derives \( \{ y \} \) in the next layer.\protect\label{fig:branching-y}}{
        {\includegraphics{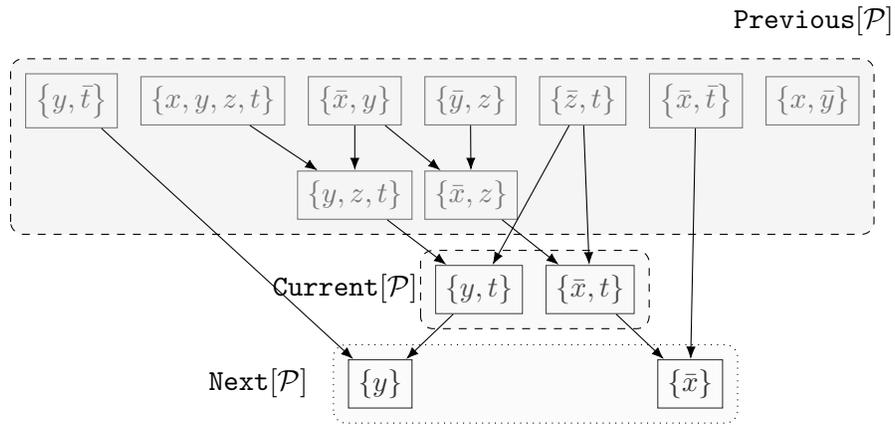}}
    }
    \subcaptionbox{Subproblem that derives \( \{ x, t \} \) but not \( \{ y \} \) in the next layer.\protect\label{fig:branching-xt}}{
        {\includegraphics{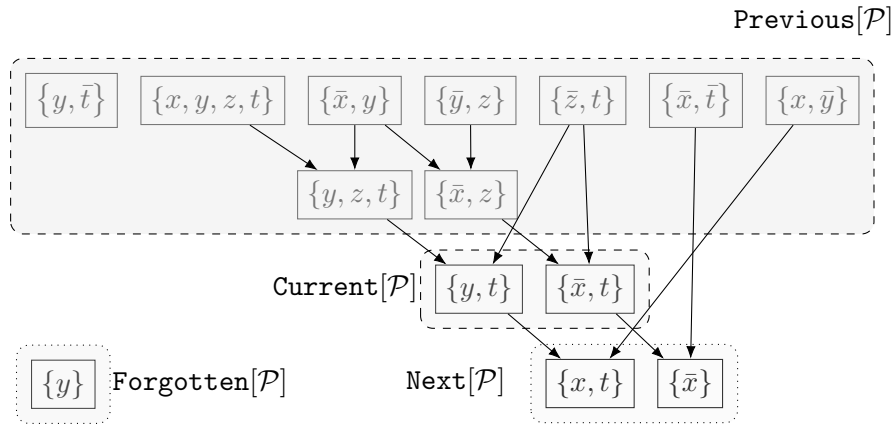}}
    }
    \subcaptionbox{Subproblem that does not derive new clauses in the next layer; \( \Next{\mathcal{P'}} = \emptyset \).\protect\label{fig:branching-commit}}{
        {\includegraphics{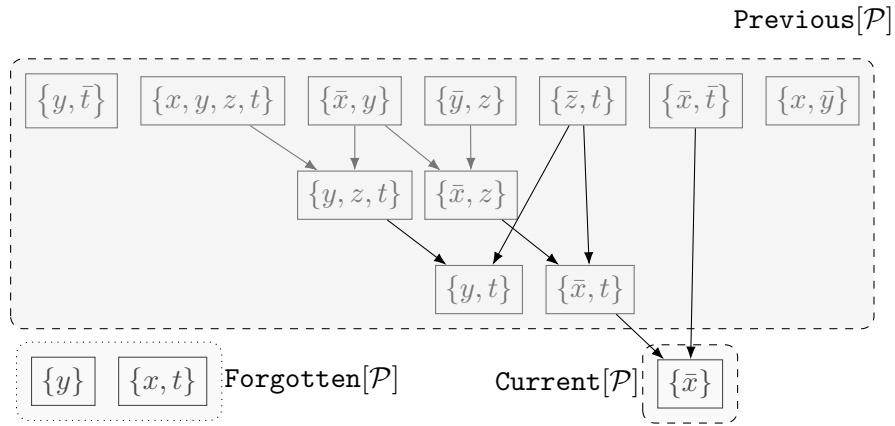}}
    }
    \caption{Subproblems that cover all the compatible proofs of the subproblem in Figure~\ref{fig:subproblem}.}
    \label{fig:subproblem-branching}
    \end{figure}
\end{example}

The next lemma formalizes this example and establishes a valid branching strategy in the sense of Proposition~\ref{proposition:correctness}.1; the proof of this fact can be found in Appendix~\ref{appendix:proofs}.

\begin{lemma}
    \label{lemma:branching}
    Let \( \mathcal{P} \) be a subproblem, and suppose that \( (\omega_1, \dotsc, \omega_m) \) is the sequence of clauses such that \begin{align*}
        \{ \omega_1, \dotsc, \omega_m \} = \{ \omega' \diamond \omega'' \mid \omega' &\in \Current{\mathcal{P}}, \omega'' \in \Current{\mathcal{P}} \cup \Previous{\mathcal{P}}, \\ \omega &\notin \Forgotten{\mathcal{P}}, \omega \notin \Known{\mathcal{P}}  \},
    \end{align*} in other words, the \( \omega_j \) lists all derivable clauses and skips forgotten clauses and clauses that are already available in the proof. Then \( \Partition{\mathcal{P}} := (\mathcal{P}'_1, \dotsc, \mathcal{P}'_m, \mathcal{P}^{commit}) \) complies with Proposition~\ref{proposition:correctness}.1 for \begin{align*}
        \Previous{\mathcal{P}'_j} &= \Previous{\mathcal{P}}, \\
        \Current{\mathcal{P}'_j} &= \Current{\mathcal{P}}, \\
        \Next{\mathcal{P}'_j} &= \Next{\mathcal{P}} \cup \{ \omega_j \}, \\
        \Forgotten{\mathcal{P}'_j} &= \Forgotten{\mathcal{P}} \cup \{ \omega_1, \dotsc, \omega_{j-1} \}
    \end{align*} and \begin{align*}
        \Previous{\mathcal{P}^{commit}} &= \Previous{\mathcal{P}} \cup \Current{\mathcal{P}}, \\
        \Current{\mathcal{P}^{commit}} &= \Next{\mathcal{P}}, \\
        \Next{\mathcal{P}^{commit}} &= \emptyset, \\
        \Forgotten{\mathcal{P}^{commit}} &= \Forgotten{\mathcal{P}} \cup \{ \omega_1, \dotsc, \omega_m \}.
    \end{align*}
\end{lemma}

\subsection{Frontier Pruning} \label{subsection:frontier}

In this subsection, we introduce a mechanism for simplifying a subproblem without loss of optimality which is based on the idea that clauses subsumed by some other known clause are not helpful for deducing unsatisfiability. We start with embedding this notion in the following definition.

\begin{definition}
    A \emph{frontier} is a set of clauses \( \Frontier{F} \) of an unsatisfiable propositional formula \( F \) that do not contain another clause of that formula: \[ \Frontier{F} := \{ \omega \in F \mid \forall \omega' \in F : \omega' \not\subset \omega \}. \] The relations \( \fin \) and \( \fsubseteq \) are shorthands for membership and inclusion in the frontier: \begin{align*}
         \omega \fin F &:\Leftrightarrow \omega \in \Frontier{F}, \\
         \Omega \fsubseteq F &:\Leftrightarrow \Omega \subseteq \Frontier{F}.
    \end{align*}
    If a proof of an UNSAT formula \( F \) only uses clauses of \( \Frontier{F} \), then this proof is called a \emph{frontier proof}.
\end{definition}

In other words, a frontier excludes obviously redundant clauses (more specifically, the ones implied by some other clause). In particular, if \( F \) is UNSAT, so is \( \Frontier{F} \), and any proof of \( \Frontier{F} \) is also a proof of the original formula \( F \). More importantly, \emph{using non-frontier clauses is pointless} in the sense that they could be swapped out with clauses from the frontier of \( \Previous{\mathcal{P}} \cup \Next{\mathcal{P}} \) to derive a stronger statement. To make this point clearer, consider the following example:

\begin{example}
    \label{example:frontier}
    Consider the subproblem obtained after adding \( \{ x, t \} \) into the next layer and choosing the ``commit" subproblem. The resulting subproblem is shown in Figure~\ref{fig:frontier}, which also shows the frontier using \( \Known{\mathcal{P}} = \Previous{\mathcal{P}} \cup \Current{\mathcal{P}} \). Note that the ``frontier clause" \( \in_{\mathcal{F}} \) property is not monotonic along the proof DAG: a frontier clause, such as \( \{ \bar{z}, t \} \), may be used to derive a clause superseded by a newer clause, which is what happens with \( \{ \bar{x}, t \} \) and \( \{ \bar{x} \} \).

    Now consider the candidates for the next layer suggested by Lemma~\ref{lemma:branching}, more specifically, the derivation \( \{ y, z, t \} =  \{ \bar{x} \} \diamond \{ x, y, z, t \} \). While this clause is valid with respect to our branching strategy, its derivation uses a non-frontier clause \(  \{ x, y, z, t \} \). If we instead use the frontier clause \( \{ x, t \} \), we derive the clause \( \{ t \} = \{\bar{x} \} \diamond \{ x, t \} \), which is a stronger version of \( \{ y, z, t \} \) in the sense that if \( \{ t \} \) is true then \( \{ y, z, t \} \) also has to be true.

    \begin{figure}[!ht]
    \centering
    \includegraphics{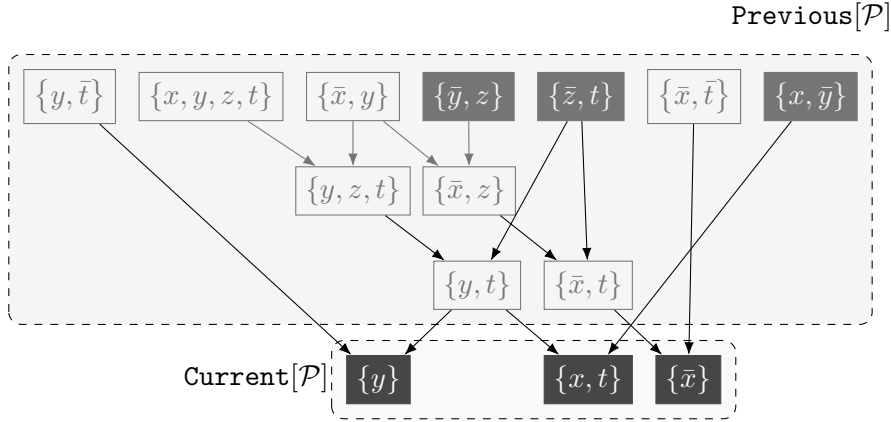}
    \caption{The frontier of the known set after deriving \( \{ x, t \} \) in Figure~\ref{fig:branching-y}; frontier clauses are highlighted with dark background and light text.}
    \label{fig:frontier}
    \end{figure}
\end{example}

The last example suggests that any proof with non-frontier clauses can be rephrased only with frontier clauses without making the proof longer -- and maybe even making it shorter. The next theorem formalizes this insight and shows by induction over the proof length that every subsequent step in the original proof can be mirrored---or even improved by removing literals---in a proof of the frontier formula; the proof can be found in Appendix~\ref{appendix:proofs}.

\begin{theorem}[Frontier sufficiency]
\label{theorem:frontier-suffiency}
    Suppose that \( F \) has a resolution proof with \( K \) resolution steps. Then the frontier \( \Frontier{F} \) has a resolution proof with at most \( K \) resolution steps.
\end{theorem}

\begin{corollary}
    \label{corollary:better-branching}
    The branching strategy from Lemma~\ref{lemma:branching} is correct when the underlying clause set is restricted to resolving the clauses from frontiers of the previous and current sets: \begin{align}
        \label{eq:action-set}
        \{ \omega_1, \dotsc, \omega_m \} = \{ \omega' \diamond \omega'' \mid \omega' &\in_{\mathcal{F}} \Current{\mathcal{P}}, \nonumber \\
        \omega'' &\in_{\mathcal{F}} \Current{\mathcal{P}} \cup \Previous{\mathcal{P}}, \\ \omega &\notin \Forgotten{\mathcal{P}}, \omega \notin \Known{\mathcal{P}} \}  \nonumber.
    \end{align}
\end{corollary}

Corollary~\ref{corollary:better-branching} implies, in particular, that the branching applied only to the frontiers remedies the problem demonstrated by Example~\ref{example:frontier}. In that case, \( \{ y, z, t \} \) can no longer be derived, as it uses the non-frontier clause; the branching strategy only creates the commit subproblem, reflecting that there is nothing meaningful to derive in the next layer.

We can combine this reasoning with a trivial observation that any derived clause has to be used somewhere later in the proof to derive the following pruning rule:

\begin{corollary}
    \label{corollary:unused}
    Any subproblem \( \mathcal{P} \) which has an unused derived clause \( \omega \notin_{\mathcal{F}} \Known{\mathcal{P}} \) can be discarded without loss of optimality.
\end{corollary}

\subsection{Pruning by Dominance} \label{subsection:dominance}

The concept of a frontier is helpful for subproblem simplification and, by extension, for pruning subproblems that have already amassed some redundancy, as in Corollary~\ref{corollary:unused}. However, we can also use this notion to introduce the dominance relationship between the subproblems which reflects the reasoning of the form ``any proof of subproblem \( P \) can be recovered in another subproblem \( \mathcal{Q} \) \emph{that has already been explored}."

\begin{definition}
\label{def:dominance}
    Let \( \mathcal{P} \) and \( \mathcal{P}' \) be the subproblems for an input formula \( F \). Then \( \mathcal{P}' \) \emph{dominates} \( \mathcal{P} \) (\( \mathcal{P}' \succcurlyeq \mathcal{P}\)) if all five of the following conditions hold: \begin{enumerate}
        \item The dominated subproblem uses all the axioms that are used by the dominating subproblem: \( \mathtt{Axioms}[\mathcal{P}'] \subseteq \mathtt{Axioms}[\mathcal{P}] \), where \( \mathtt{Axioms}[\mathcal{P}] = F_{used} \cup F_{corr} \) is the union of axioms used in some derivations \( F_{used} \), and \( F_{corr} \subseteq F \) is the ``correcting'' set of axioms such that removing them renders the formula satisfiable: \( \omega \in F_{corr} \implies F \setminus \{ \omega \} \text{ is SAT}\).
        \item The dominating subproblem can derive the clauses that can be derived in the dominated subproblem: \( \mathtt{Derivable}[\mathcal{P}'] \supseteq \mathtt{Derivable}[\mathcal{P}] \), where \( \mathtt{Derivable}[\cdot] \) is introduced in Eq.~\eqref{eq:action-set}.
        \item The frontier of the dominating subproblem contains the frontier of the dominated subproblem: \(\Frontier{\Known{\mathcal{P'}}} \supseteq \Frontier{\Known{\mathcal{P}}}. \)
        \item Any clause forgotten in the dominating subproblem is forgotten in the dominated subproblem \( \Forgotten{\mathcal{P}'} \subseteq \Forgotten{\mathcal{P}} \).
        \item The dominating subproblem contains no more resolution steps than the dominated subproblem: \( \# \Known{\mathcal{P}'} \le \# \Known{\mathcal{P}} \).
    \end{enumerate}
\end{definition}

\begin{note}
    If the input formula \( F \) is \emph{minimally} unsatisfiable, the axiom condition trivially holds, as in this case all axioms must be used in a proof.
\end{note}

We can now use the dominance relationship to implement the \( \mathtt{Lookup} \) subroutine as a search of the dominated subproblem among the previously visited subproblems \( \mathcal{D} \). Lemma~\ref{lemma:lookup} defines this procedure, and the proof of its correctness can be found in Appendix~\ref{appendix:proofs}.

\begin{lemma}
\label{lemma:lookup}
    \( \Lookup{\mathcal{P}}{\mathcal{D}} := \exists \mathcal{P}' \in \mathcal{D} : \mathcal{P}' \succcurlyeq \mathcal{P} \) complies with Proposition~\ref{proposition:correctness}.3. 
\end{lemma}

\subsection{Lower Bounding the Proof Length with UNSAT Clause Subsets} \label{subsection:lower-bound}

To complete the definition of the branch-and-bound procedure of Algorithm~\ref{alg:branch-bound-and-commit}, we need to specify a valid subroutine for deriving proof length lower bounds. We start this by restating the lower bound on the proof length of a \emph{minimally unsatisfiable formula} from \citeA{Peitl2021-jArtifIntellRes}:

\begin{lemma}[Lemma 4, \citeA{Peitl2021-jArtifIntellRes}]
\label{lemma:bounding-mus}
    Suppose that \( F \) is a minimally unsatisfiable formula (i.e., excluding any clause from \( F \) makes it satisfiable). Then any proof of \( F \) has at least \( 2\#F - 1 \) clauses: all \(\#F \) formula clauses and at least \( \#F - 1 \) derived clauses.
\end{lemma}

The key idea of our bounding approach is to generalize this bound to arbitrary formulas. The next lemma presents the resulting expression; the proof can be found in Appendix~\ref{appendix:proofs}.

\begin{lemma}
\label{lemma:bounding-mus-enum}
    Any compatible proof of a subproblem \( \mathcal{P} \) derives at least \[ \min\{ \#S - 1 \mid S \fsubseteq \Known{\mathcal{P}}, S \text{ is UNSAT} \} \] clauses not present in \( \mathcal{P} \).
\end{lemma}

This idea can be strengthened using Corollary~\ref{corollary:unused} to only consider clause subsets \( S \) that include all unused clauses:

\begin{lemma}
\label{lemma:bounding-mus-enum-fixed}
    Let  \( \mathcal{P} \) be a subproblem and \( U \fsubseteq \Known{\mathcal{P}} \) be the set of clauses that have not been used as a premise of a resolution step. Consider a compatible proof \( L \) that does not use all of the clauses from \( U \). Then there is another subproblem \( \mathcal{P}' \) with a compatible proof \( L' \) having fewer clauses than \( L \).
\end{lemma}

We can use the same reasoning to bound the number of \emph{axioms} used in the proof, as any proof has to introduce an unsatisfiable subset of the input formula, and the cardinality of the smallest such set is a valid lower bound on the used axioms. We first define the problem of finding the smallest minimal unsatisfiable subset as follows: \begin{definition}
    Suppose that \( F \) is an UNSAT formula and \( U \subseteq F \) is the set of \emph{fixed} clauses. Then the \emph{smallest minimal unsatisfiable subset} is defined as the set that has the smallest cardinality among the unsatisfiable subsets of \( F \) containing \( U \) as a subset: \[ \SMUS(F; U) := \arg\min\{ \# S \mid U \subseteq S \subseteq F, S \text{ is UNSAT} \}. \]
\end{definition}

\begin{note}
    As soon as \( U \fsubseteq F \), we can assume without loss of generality that any MUS of \( F \) can be transformed to a MUS of \( \Frontier{F} \) without increasing the cardinality; in particular, under that assumption, we have \(  \SMUS(F; U) =  \SMUS(\Frontier{F}; U) \).
\end{note}

Suppose first that the input formula \( F \) is minimally unsatisfiable. In that case, every formula clause has to be used at least once; the same applies to the derived clauses, as they can be discarded otherwise. Then we can underestimate the proof length by counting the clauses derived in \( \Known{\mathcal{P}} \) and applying Lemma~\ref{lemma:bounding-mus-enum} to bound the number of remaining clauses.

\begin{corollary}
    Given a \emph{minimally unsatisfiable} formula \( F \), let \( \mathcal{P} \) be a subproblem with \( U \subseteq \Known{\mathcal{P}} \) being the set of clauses, both from the formula and the derivations, that were not used so far. Then the expression \begin{equation}
        \label{eq:mus-lower-bound}
        \#\Known{\mathcal{P}} + \#\SMUS(\Known{\mathcal{P}}; U) - 1
    \end{equation} satisfies Proposition~\ref{proposition:correctness}.2 as a lower bound on the length of compatible proofs.
\end{corollary}

For the general case when \( F \) is not necessarily a MUS, we can rework Eq.~\eqref{eq:mus-lower-bound} with the previous observations together to introduce the following bounding function:

\begin{corollary}
    Given a formula \( F \), let \( \mathcal{P} \) be a subproblem with \( U \subseteq \Known{\mathcal{P}} \) being the set of derived clauses not used so far. Suppose that \( F_{used} \subseteq F \) is the set of axioms used in some derivations, and \( F_{corr} \subseteq F \) is the set of axioms such that removing them renders the formula satisfiable: \( \omega \in F_{corr} \implies F \setminus \{ \omega \} \text{ is SAT}\). Then the expression \begin{equation}
        \label{eq:lower-bound}
        \#\SMUS(F; F_{used} \cup F_{corr}) + \#(\Known{\mathcal{P}} \setminus F) + \#\SMUS(\Known{\mathcal{P}}; U) - 1
    \end{equation} satisfies Proposition~\ref{proposition:correctness}.2 as a lower bound on the length of compatible proofs.
\end{corollary}

The difference in this case from the MUS expression in Eq.~\eqref{eq:mus-lower-bound} is that we can no longer assume that all formula clauses are going to find their way into the shortest proof. The expression in Eq.~\eqref{eq:lower-bound} accounts for this by splitting the \( \#\Known{\mathcal{P}} \) into two parts: one that counts derived clauses, which are all known clauses not in \( F \), and the other that underestimates the number of used formula clauses given that (a) used clauses can no longer be un-used, and (b) clauses in \( F_{corr} \) cannot be excluded without making the formula SAT. As the proof still needs to choose an UNSAT subset of \( F \), Eq.~\eqref{eq:lower-bound} incorporates this with the SMUS term as the first summand.

Unfortunately, determining the \( \SMUS(\cdot, \cdot) \) quantity is itself a non-trivial task; in fact, it is known to be \( \Sigma_2^P \)-complete \cite{Liberatore2005-artifIntell}. Therefore, we also need an efficient procedure for determining SMUS lower bounds; we proceed to use them transitively in Eq.~\eqref{eq:mus-lower-bound} or Eq.~\eqref{eq:lower-bound}. To this end, we use the Forqes package by \shortciteA{Ignatiev2015-Other}; the defining trait of this approach is that it is a dual optimization algorithm in the sense that it enumerates lower bounds on the SMUS cardinality until it discovers the unsatisfiable set having exactly that length. Therefore, we run this approach with a time limit of 1 second for each of the two SMUS calls and introduce the best discovered lower bound instead of the SMUS terms in Eq.~\eqref{eq:mus-lower-bound} and Eq.~\eqref{eq:lower-bound}.

However, for subproblems with small frontier sets, this approach introduces an overhead that does not substantially slow down a \emph{single} bounding call but accumulates as the algorithm explores different subproblems, as it has to make multiple calls to the SMUS bounding subroutine. To remedy this, we run a branch-and-bound algorithm for finding the SMUS with the same time limit if the SMUS subroutine input has at most \( M_{switch} \) clauses.

More specifically, the branch-and-bound approach represents the subproblems by the set of \emph{formula} clauses and the set of \emph{mandatory} clauses; the root subproblem corresponds to the inputs of a \( \SMUS(\cdot, \cdot) \) call. The lower bound is evaluated as the number of mandatory clauses, and the branching step corresponds to taking a non-mandatory clause \( \omega \) and creating two subproblems as follows: \begin{itemize}
    \item retain all the formula clauses, add \( \omega \) to the mandatory set;
    \item remove \( \omega \) from the formula clauses and add all correcting clauses to the mandatory set; the correcting clause here means that removing it from the formula clauses makes the remaining set satisfiable.
\end{itemize}

\section{Experimental Evaluation} \label{section:experiments}

As our approach can be seen as both the procedure that eventually discovers the shortest proof and as an anytime optimization algorithm operating on resolution proofs, we are able to put forward two distinct experimental hypotheses for each of those viewpoints as follows: \begin{itemize}
    \item Our approach leverages the layer list representation to reduce the time to arrive at the shortest proofs.
    \item Our approach can enumerate the sets of implied clauses to give enough information to CaDiCaL to discover shorter proofs.
\end{itemize} In this section, we conduct an extensive empirical evaluation of our approach addressing both our hypotheses: we compare the produced proof length against CaDiCaL on competition formulas and the time to optimality against the baseline approach from \citeA{Peitl2021-jArtifIntellRes} on small, synthetic formulas. Our evaluation supports both hypotheses: for the smaller formulas, our approach is capable of discovering optimal proofs faster (by orders of magnitude in many cases), whereas, for more realistic formulas, our approach discovers substantially shorter proofs than CaDiCaL can discover on its own. An important caveat here is that SAT solvers are designed to produce \emph{faster}---but not necessarily shorter---proofs of unsatisfiability; while these two quantities are closely related, shorter proofs do not automatically imply lower runtime.

The implementation of our approach with the infrastructure for running the experiments is available in Online~Appendix~1.\footnote{Available at \url{https://dx.doi.org/10.5281/zenodo.14051922}.} We ran our experiments on the \emph{DelftBlue}~\cite{DelftHighPerformanceComputingCentre2024-Other} supercomputer, with each experimental run being allocated a single core of an Intel Xeon E5-6248R 24C 3.0GHz processor and 64GB of RAM with a time limit of one hour.

\subsection{Algorithm Configuration}

Aside from the design choices described in Section~\ref{section:proofmin}, we have to choose the heuristics for the following algorithm components: \begin{description}
    \item[Completion backend] Aside from using CaDiCaL to implement the \( \mathtt{Complete} \) subroutine of Algorithm~\ref{alg:branch-bound-and-commit}, we also provide our implementation of the DPLL procedure by \citeA{Davis1960-jAcm} and \shortciteA{Davis1962-communAcm}.
    \item[Randomization of \( \mathtt{Complete} \) procedure] We consider two options for the seed value of CaDiCaL solver used to complete the proofs, namely, a \emph{static seeding}, where {CaDiCaL} receives a seed value of 0 on every call, and \emph{dynamic seeding}, where the seed value is itself sampled randomly. Our motivation behind the dynamic seeding is that randomizing the SAT solver executions over many similar instances leads to a more diverse exploration of the space of valid proofs.
    \item[Clause ordering strategy] This part determines the order in which the clauses are added in Lemma~\ref{lemma:branching} to the forgotten set and/or the next layer.
    \item[SMUS switch-point value \( M_{switch} \)] As mentioned in Section~\ref{subsection:lower-bound}, this is the highest number of clauses for which we invoke the brute-force approach instead of Forqes.
    \item[Dominance cache clean-up] As we expect most of the dominance cache lookups (Section ~\ref{subsection:dominance}) are not going to find a dominating subproblem, we also investigate strategies for removing outdated dominance cache entries.
\end{description}

Depending on the use case, we use two different sets of heuristics. More precisely, we consider \emph{optimality-focused} runs, where we seek to produce the optimal proof as fast as possible, and \emph{length-focused} runs, where we look for the shortest possible proof by the time limit, regardless of the lower bound on the proof length. The configurations that we use for these use cases are as follows: \begin{itemize}
    \item Optimality-focused configuration: DPLL completion with dynamic seeding, clause ordering by descending clause length and descending literal frequency, switch to brute-force SMUS bounding at \( M_{switch} = 28 \), discard dominance cache entries that have not been accessed on the last \( T_{\max} = 10^5 \) iterations.
    \item Length-focused configuration: CaDiCaL completion with dynamic seeding, clause ordering by ascending clause length, and descending literal frequency, with the other parameters being the same as for the optimality-focused configuration.
\end{itemize}

See Appendix~\ref{appendix:configuration} for the empirical justification behind the choices of parameter values.

\subsection{Dataset Description}

We run our evaluation on two collections of UNSAT formulas available in Online~Appendix~2:\footnote{Available at \url{https://dx.doi.org/10.5281/zenodo.14052283}.} \begin{itemize}
    \item \emph{Synthetic} instances are the formula obtained from some procedure with input parameters influencing the formula size. Random 3-CNFs are one example, as we can specify both the number of clauses and the number of variables upfront.
    \item \emph{Competition} instances are the 251 user-submitted instances from the SAT Competition 2002 collection described by \shortciteA{Simon2005-annMathArtifIntell} that have been proven unsatisfiable within 15 seconds. The choice of that edition of the SAT Competition was motivated by our need for a benchmark realistic enough to correspond to common use cases of SAT solvers but simple enough to be manageable within our approach.
\end{itemize}

The synthetic part of the evaluation dataset comprises the following instances:
\begin{itemize}
    \item \emph{Pigeonhole principle} formulas for \( 1 \le n \le 6 \) holes and \( m = n + 1 \) pigeons.
    \item \emph{Parity principle} formulas for \( (2n-1) \) elements with \( 1 \le n \le 6 \).
    \item \emph{Ordering principle} formulas for \( 1 \le n \le 6 \) elements.
    \item \emph{Random 3-CNFs} with the number of variables \( n \) varying from 5 to 15 and the number of clauses \( m \) varying from \( 4n \) to \( 6n \).
    \item \emph{Subset cardinality} formulas for random bipartite graphs with the part size \( n \) varying between 5 and 15.
    \item \emph{Graph coloring} formulas encoding \( k \)-colorability of random \( 2(k-1) \)-regular graphs with \( k \) varying from 3 to 5 and number of vertices varying from \( 2k \) to \( 4k\).
    \item \emph{Graph coloring with clique} formulas are the same but with a random \( (k+1) \)-clique introduced in the graph, thereby guaranteeing the unsatisfiability of a formula.
    \item \emph{Saturated minimally unsatisfiable} formulas obtained with the procedure described by \citeA{Peitl2021-jArtifIntellRes} for at most 9 clauses.
\end{itemize}
Our motivation behind this choice of instances was to include instance classes that have many small UNSAT formulas, are simple enough to generate, and are diverse enough to evaluate our approach in different situations.

For both collections, we also create a MUS version of each instance collection (unless the source is already a MUS by design) by invoking the MUSer2 package by \citeA{Belov2012-jSatisfBooleanModelComput} on the input formulas. We extend the instance collection in this manner due to an empirical observation that the MUSes of the original formulas are harder for SAT solvers to solve, hinting at a distribution of instances substantially different from the original formulas. In addition, the description of the approach also suggests that a MUS input carries more structure for the proof minimizer, namely, that all formula clauses have to be used at some proof inference.

\subsection{Reference Proofs}

We start with discussing the methodology we used for extracting the resolution proofs from a SAT solver and measuring their lengths. Similarly to the \( \mathtt{Complete} \) procedure implementation from Section~\ref{subsection:algorithm}, we use CaDiCaL 2.0 solver with UNSAT mode enabled. We also request the solver to produce the proofs in LRAT format; the choice of the proof format was dictated by the fact that it contains not only derived clauses but also \emph{the clauses used during the unit propagation}, which is essential for tracing the resolution sequence leading to a given learned clause.

To compute the proof length, we parse the proof, ignore the deletion lines, and simulate the resolution steps needed to derive a clause by going through the propagated clauses in reverse order, which is guaranteed to correspond to valid resolution steps due to the LRAT format structure. We also track the clauses derived in resolution steps \emph{and skip the previously visited clauses}, regardless of whether they were mentioned in the proof file. Given that the deduplication step may considerably influence the final proof length, we further investigate its importance.

\begin{figure}[!ht]
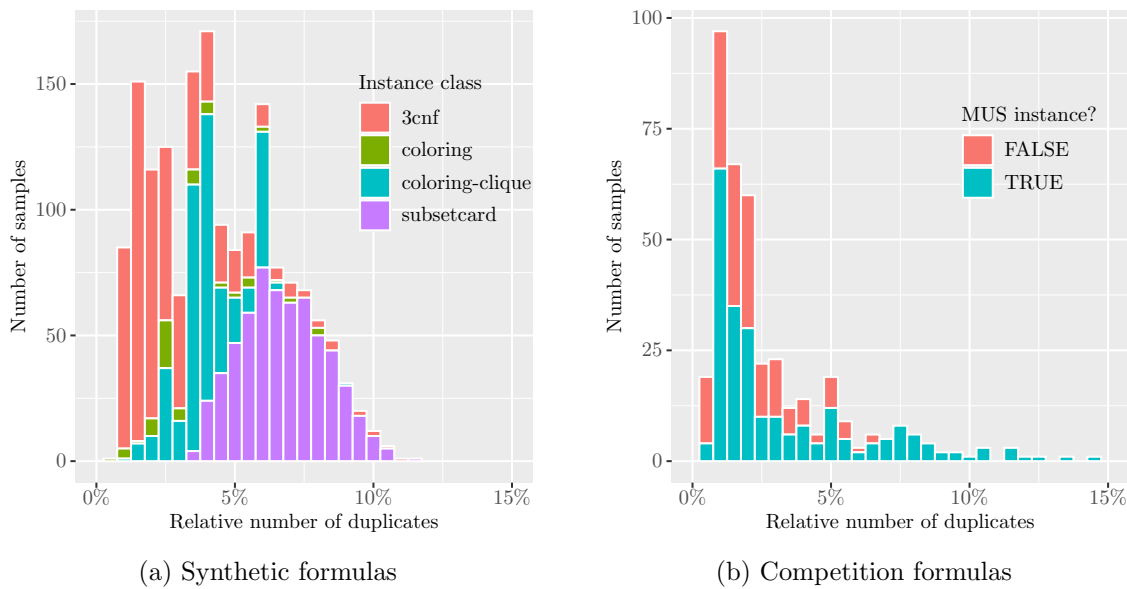

\centering
\subcaptionbox{Synthetic formulas\protect\label{fig:relative-dups}}[0.48\textwidth]{
    \resizebox{0.48\textwidth}{!}{
    \includesvg{relative-dups}}
}
\hfill
\subcaptionbox{Competition formulas\protect\label{fig:relative-dups-sat}}[0.48\textwidth]{
    \resizebox{0.48\textwidth}{!}{
    \includesvg{relative-dups-comp}}
}
\caption{Distribution of the fraction of duplicate derivations from the proof length with duplicated steps.}
\label{fig:dup-influence}
\end{figure}

Figure~\ref{fig:relative-dups} suggests that the influence of deduplication is severely limited for the synthetic formulas, with the subset cardinality formula class being the only formula type for which a non-trivial amount of proofs has been reduced by at least 6\%. The duplication effect is more pronounced for some of the competition formulas, particularly their MUS variants, as shown in Figure~\ref{fig:relative-dups-sat}, with some of the instances yielding as many as 15\% of duplicated inferences. However, typical competition formulas are even less prone to the duplication of clauses than synthetic formulas; for this reason, we use the proof lengths without the duplication as the reference values in the rest of the experimental evaluation.

\subsection{Resolution Number Computation}

We start the evaluation of our approach by reproducing the procedure for computing the \emph{\(m\)-th resolution hardness number}, which is defined by \citeA{Peitl2021-jArtifIntellRes} as the largest number of clauses necessary to prove the unsatisfiability of a formula with \( m \) clauses. In our notation, this can be formulated as \( h_m := \max_{F : \# F = m} \min \{ \# R \mid R \in \mathtt{Proofs}[F] \} \).

The approach from \citeA{Peitl2021-jArtifIntellRes} improves upon the naïve application of the \( h_m \) definition by: \begin{itemize}
    \item replacing the set of all \( m \)-clause formulas with a set of \emph{saturated minimally unsatisfiable} formulas,
    \item enumerating the equivalence classes of these formulas up to variable permutations,
    \item and evaluating the proof lengths for each of those formulas.
\end{itemize}

We introduce the formulas with up to 9 clauses obtained by this procedure into our dataset and report the results of both approaches on this subset of our collection. Table~\ref{table:hardness-number-time} shows the time\footnote{The values of \( h_1, h_2, h_3 \) are trivially proven without any explicit reasoning on the shortest proof problem.} to compute \( h_m \) values for \( 4 \le m \le 8 \) using the baseline approach, where the final step is offloaded to a SAT solver and our approach in optimality-focused configuration. We can confirm that the first 8 values of \( h_m \) computed with our approach and with the baseline are the same; additionally, we reduce the computational time spent on solving shortest proof problem instances by a factor of~7. In addition, we also ran the same evaluation for a subset of 9-clause formulas that we managed to enumerate, suggesting that the divergence between the runtimes of the approaches grows with the formula size.

\begin{table}
\centering
\begin{tabular}{@{}lllllll@{}}
\toprule
\multicolumn{1}{c}{Number of clauses \( m \)} & 4 & 5 & 6 & 7 & 8 & 9* \\ \midrule
Total duration, our approach (s)          & 0.06 & 0.05 & 0.32 & 1.51 & 16.9 & 289 \\
Total duration, baseline (s)              & 0.39 & 0.35 & 1.55 & 7.72 & 118 & 3424 \\ \bottomrule
\end{tabular}
\caption{Total execution time for solving the proof minimization problem for \( m \)-clause saturated minimally unsatisfiable formulas. The column with an asterisk corresponds to an \emph{incomplete} set of formulas.}
\label{table:hardness-number-time}
\end{table}

While these results are reassuring, they also mean that the \emph{enumeration} of candidate formulas is the primary hurdle to computing the next hardness numbers. For example, we have observed that even the task of enumeration of all 9-clause saturated minimally unsatisfiable formulas is prohibitively time-consuming, and the situation rapidly deteriorates as the formula size grows. For example, even producing the set of 5-variable formulas with 10 clauses requires 50 CPU hours, with the generation of subsets with higher variable counts being terminated after exhausting our 600 CPU hour budget.

\subsection{Evaluation on Synthetic Formulas}

We start by discussing the instances depending on a single parameter, namely, the pigeonhole formulas, parity formulas, and ordering principle formulas. Tables~\ref{table:parametric-time-php}, \ref{table:parametric-time-parity}, \ref{table:parametric-time-op-orig}, and \ref{table:parametric-time-op-mus} compare the proof lengths produced on these formulas by CaDiCaL with the proofs we have discovered with our approach in length-focused configuration, highlighting that even for these well-known formulas there is some room for improvement.

\begin{table}[ht]
\centering
\begin{tabular}{@{}lllllll@{}}
\toprule
\multicolumn{1}{c}{Number of pigeons \( n \)} & 1 & 2 & 3 & 4 & 5 & 6 \\ \midrule
Proof length, our approach                  & 5* & 19* & 66  & 277 &  1378 & 10420 \\
Proof length, CaDiCaL                      &   5 & 20  & 79  & 363  &  1793 & 14320 \\ \bottomrule
\end{tabular}
\caption{Proof lengths for pigeon-hole principle formulas. Values with asterisks are additionally proven to be optimal.}
\label{table:parametric-time-php}
\end{table}

\begin{table}[ht]
\centering
\begin{tabular}{@{}llllll@{}}
\toprule
\multicolumn{1}{c}{Number of elements \( n \)} & 1 & 2 & 3 & 4 & 5 \\ \midrule
Proof length, our approach                  & 11* & 81 & 513  &  4517 &  55247 \\
Proof length, CaDiCaL                      &  12 & 86  & 601  & 5161  &  60176 \\ \bottomrule
\end{tabular}
\caption{Proof lengths for parity principle formulas. Values with asterisks are additionally proven to be optimal.}
\label{table:parametric-time-parity}
\end{table}

\begin{table}[ht]
\centering
\begin{tabular}{@{}llllll@{}}
\toprule
\multicolumn{1}{c}{Number of elements \( n \)} & 1 & 2 & 3 & 4 & 5 \\ \midrule
Proof length, our approach                  & 5* & 16* & 39  & 77 & 134 \\
Proof length, CaDiCaL                      &  5 & 21  & 44  & 81  & 142 \\ \bottomrule
\end{tabular}
\caption{Proof lengths for ordering principle formulas. Values with asterisks are additionally proven to be optimal.}
\label{table:parametric-time-op-orig}
\end{table}

\begin{table}[ht]
\centering
\begin{tabular}{@{}llllll@{}}
\toprule
\multicolumn{1}{c}{Number of elements \( n \)} & 1 & 2 & 3 & 4 & 5 \\ \midrule
Proof length, our approach                  & 5* & 16* & 40  & 79 & 273 \\
Proof length, CaDiCaL                      &  5 & 18  & 70  & 223  &  561 \\  \bottomrule
\end{tabular}
\caption{Proof lengths for MUSes of ordering principle formulas. Values with asterisks are additionally proven to be optimal.}
\label{table:parametric-time-op-mus}
\end{table}

Among the other synthetic formulas, the only classes that have been solved to optimality are saturated MUSes and random 3-CNFs. We thus proceed to compare the performance of our approach on these instances against the approach by \citeA{Peitl2021-jArtifIntellRes}. To this end, we compare (a) the sets of instances solved by either approach and (b) the time it took both approaches to solve an instance to optimality within the time limit.

\begin{wrapfigure}{r}{0.45\textwidth}
    \centering
    \includesvg[width=0.4\textwidth]{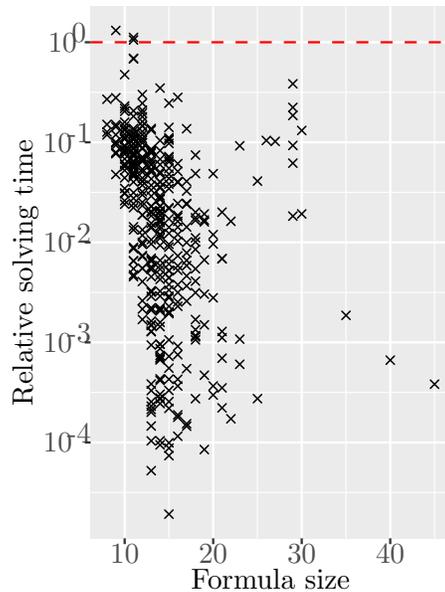}
    \caption{Distribution of the ratio of solving time with our approach to the baseline solving time; lower is better.}
    \label{fig:3cnf-mus-relative-time}
\end{wrapfigure}

For the number of solved instances, we observe that there is only \emph{one} instance that our approach has failed to solve to optimality but was solved by the baseline. On the other hand, our approach has solved 420 extra MUS instances on top of 405 MUS instances solved by both approaches. (The comparison on non-MUS instances is much more biased, with 16 instances solved by both approaches and 936 instances solved only with our approach.) Figure~\ref{fig:3cnf-mus-relative-time} shows the relative time to solve an instance for the samples solved by both approaches with respect to the number of clauses in a formula. The only instances where our approach had to take more time correspond to runs where both approaches finished in a fraction of a second; for the remaining instances, differences in orders of magnitude are not uncommon.

We have investigated the factors influencing the time to optimality, and our results suggest the exponential trend depending on the formula size and on the \emph{MUS bound gap}, which for the case of MUS formulas is defined as \( \# R^*(F) - (2\# F - 1) \), the gap between the shortest proof and the MUS lower bound from Lemma~\ref{lemma:bounding-mus-enum}. Figure~\ref{fig:3cnf-mus-bnb-time-gap} supports the claim about the MUS bound gap; as can be seen, the solving time scales exponentially with the increasing gap to MUS bound. On the other hand, Figure~\ref{fig:3cnf-mus-bnb-time-clauses} suggests that the search time is exponential, with the base determined by the MUS bound gap and the exponent defined by the formula size.

Another pattern that can be seen is the substantial performance difference between instances with zero MUS bound gap shown in Figure~\ref{fig:3cnf-mus-bnb-time-clauses-zero} from the other instances shown in Figure~\ref{fig:3cnf-mus-bnb-time-clauses-nonzero}. We believe that the search in the former case depends on how fast the right proof is discovered but not on the search tree structure, which also justifies randomization in proof completion. On the other hand, the instances from the latter group can only be considered solved once all \emph{subproblems} with the bound equal to the root subproblem bound are enumerated.

\begin{figure}[ht]
    \centering
    \includesvg[width=0.9\textwidth]{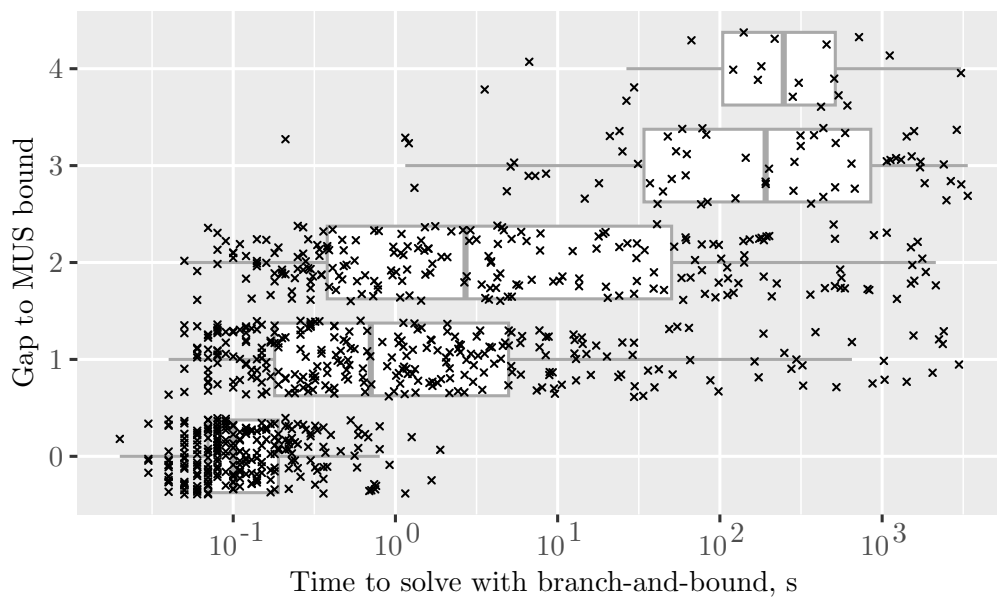}
    \caption{Distribution of the solving time of our approach by gaps to the MUS bound. A vertical jitter is added for clearer visualization.}
    \label{fig:3cnf-mus-bnb-time-gap}
\end{figure}

\begin{figure}
    \centering
    \subcaptionbox{Instances with zero MUS bound gap\protect\label{fig:3cnf-mus-bnb-time-clauses-zero}}{
        \includesvg{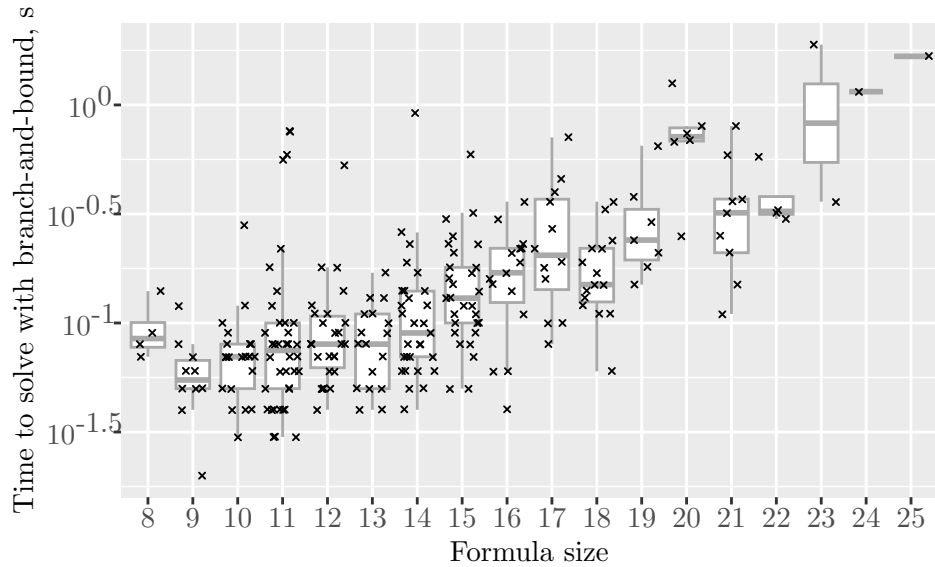}
    }
    \subcaptionbox{Instances with a nonzero MUS bound gap\protect\label{fig:3cnf-mus-bnb-time-clauses-nonzero}}{
        \includesvg{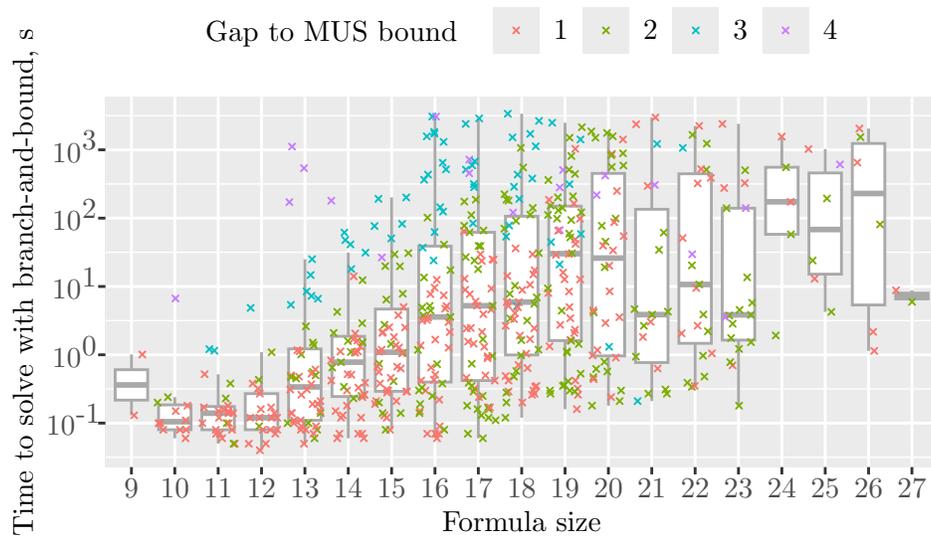}
    }
    
    \caption{Distribution of the solving time of our approach by formula size. A horizontal jitter is added for clearer visualization.}
    \label{fig:3cnf-mus-bnb-time-clauses}
\end{figure}

We have also compared the proof lengths across the benchmark formulas between our approach and CaDiCaL; Figure~\ref{fig:bnb-vs-cadical-length} summarizes this comparison. Our approach is consistently able to reduce the proofs by CaDiCaL by 25\%, with some of the proofs being halved in progress; the comparison is particularly striking for subset cardinality formulas, where our approach consistently finds reductions of at least 40\%.

\begin{figure}[!ht]
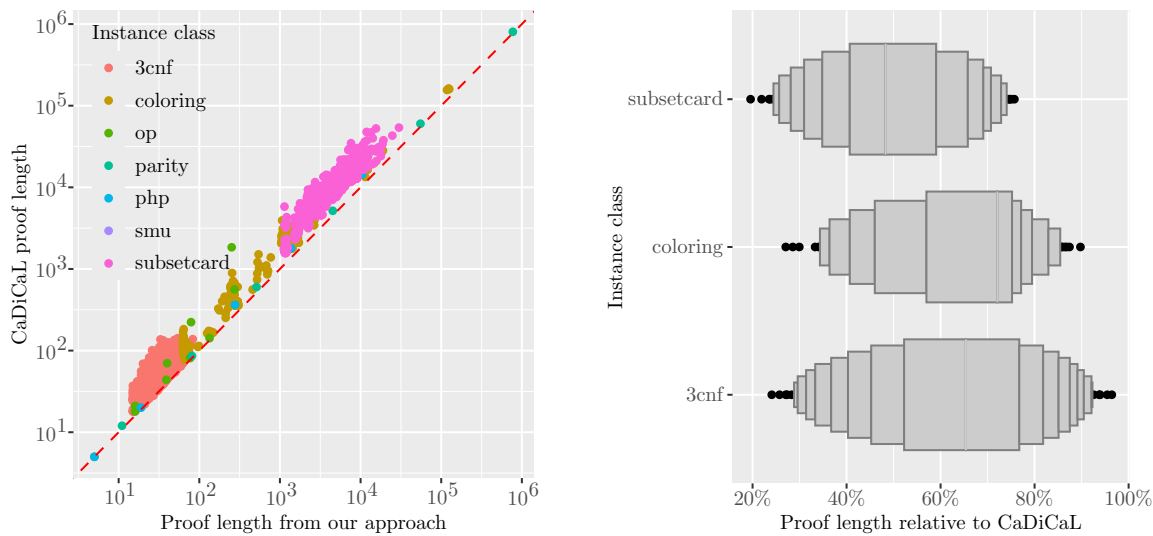

\centering
\subcaptionbox{Proof length of our approach vs. CaDiCaL proof length.}[0.48\textwidth]{
    \resizebox{0.48\textwidth}{!}{
    \includesvg{length-scatter}}
}
\hfill
\subcaptionbox{Ratio of our proof length to CaDiCaL proof length.}[0.48\textwidth]{
    \resizebox{0.48\textwidth}{!}{
    \includesvg{length-box}}
}
\caption{Comparison of proofs produced by our approach with CaDiCaL proofs.}
\label{fig:bnb-vs-cadical-length}
\end{figure}

\subsection{Evaluation on Competition Formulas}

Finally, we evaluate the performance of our approach on the UNSAT instances from SAT Competition 2002; we use the length-focused version of the approach with several changes:
\begin{itemize}
    \item We disable the randomization in order to evaluate how well our procedure is able to navigate the proof space without artificially diversifying it.
    \item We only consider the top 10 clauses after applying the ordering criterion in branching, as it is unreasonable to expect the completeness of this procedure.
    \item We disable the pruning of the subproblems, as the lower-bound pruning is several orders of magnitude away from the discovered proof lengths at best and times out without any meaningful bound at worst, whereas the dominance pruning does not yield any substantial advantage on larger formulas.
    \item We limit the queue size to \( 10^4 \), discarding the subproblems with the largest ``trivial" lower bound (length of the shortest clause), and earlier subproblems in case of ties.
\end{itemize} 

Figure~\ref{fig:competition-length} compares the CaDiCaL proof lengths with the proof lengths discovered by our approach. Typically,\footnote{We report the range between the first and third quartiles.} our approach reduces the proof length by 30---65\% of the proof length on the original competition instances and 15---50\% on their MUS versions. We also several instances where our approach reduces the proof length by at least a factor of 10: \begin{itemize}
    \item four FPGA routing formulas: \texttt{apex7\_gr\_2pin\_w4}, \texttt{c880\_gr\_rcs\_w6}, \texttt{term1\_gr\_2pin\_w3}, and \texttt{vda\_gr\_rcs\_w7},
    \item seven encodings of factorization circuits: \texttt{ezfact16\_\{1,2,3,5,8,9,10\}},
    \item and three instances produced from the formulas by \citeA{Urquhart1987-jAcm}: \texttt{urquhart2\_25}, \texttt{Urquhart-s3-b1}, and \texttt{Urquhart-s3-b9}.
\end{itemize}

\begin{figure}
    \centering
    \includesvg[width=1.0\linewidth]{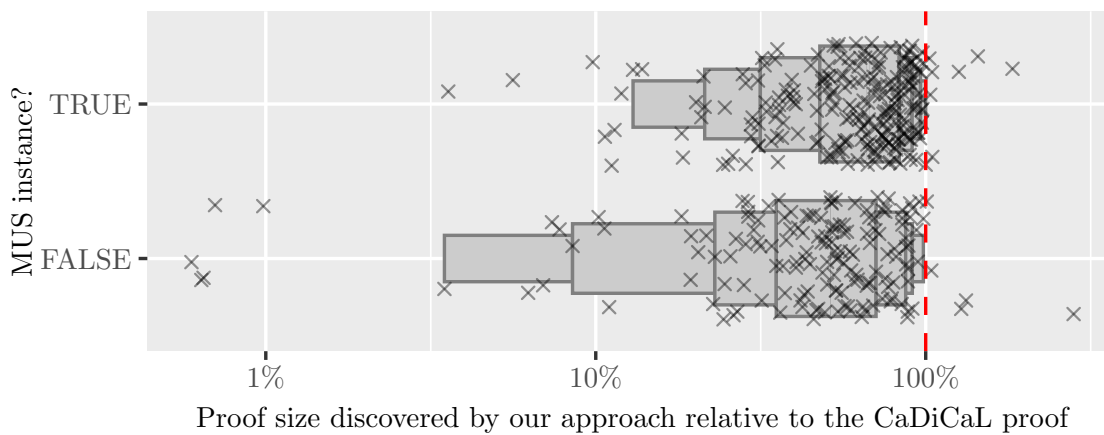}
    \caption{Distribution of the ratios of proof lengths produced with our approach to CaDiCaL proof lengths. Left is better; a vertical jitter is added for clearer visualization.}
    \label{fig:competition-length}
\end{figure}

However, our approach also crashes in 22 out of 414 instances (\(\approx 5\%\)) with out-of-memory errors, with two original SAT Competition instances and the remaining 20 being their MUS versions. These instances correspond to long CaDiCaL proofs, which translates into using large amounts of memory to unpack a single run of the solver. Figure~\ref{fig:competition-limits} illustrates this problem by comparing the CaDiCaL proof length for the instances where our approach has successfully terminated with the ones where our approach ran out of memory. The issues begin when the solver generates around \( 10^6 \) proof steps to prove the unsatisfiability, highlighting the limitations of using resolution proofs during the proof minimization.

\begin{figure}
    \centering
    \includesvg[width=1.0\linewidth]{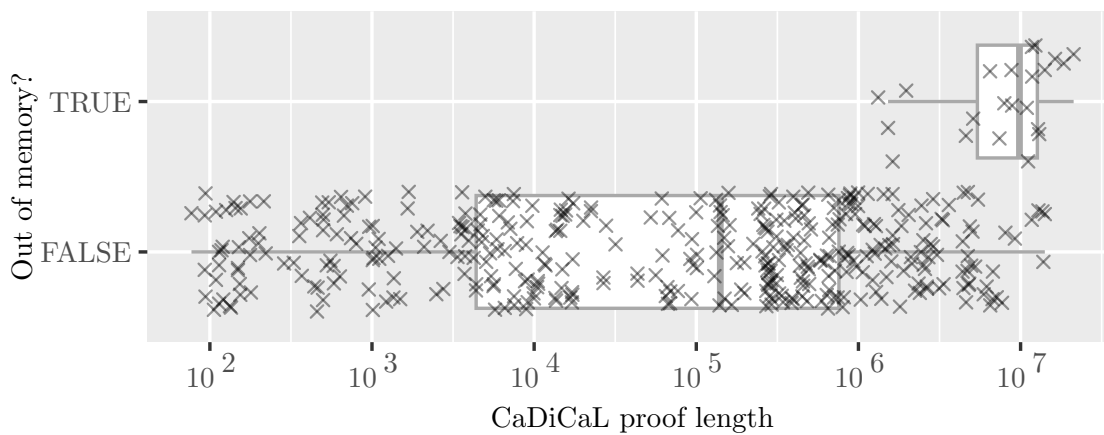}
    \caption{Distribution of the CaDiCaL proof lengths for failed and successful instances. A vertical jitter is added for clearer visualization.}
    \label{fig:competition-limits}
\end{figure}

\section{Conclusions} \label{section:conclusions}

We propose a novel approach for constructing the shortest resolution proof of a given unsatisfiable propositional formula. In its anytime form, this approach can also discover increasingly tight upper bounds (shorter proofs) and lower bounds (higher under-estimates of the proof length). The proposed approach has derived substantially shorter proofs on a wide range of instances; it has also managed to prove optimality much faster than the baseline approach for formulas with up to 45 clauses, often exhibiting speed-ups of orders of magnitude.

To achieve this, we also put forward several algorithmic improvements. First, we propose the layer list representation, which is uniquely defined for any set of clauses constituting a proof; in addition, this representation has a stepwise structure, which serves as the foundation of the branch-and-bound algorithm. To support it, we propose three distinct pruning procedures. The lower bound pruning reasons over MUSes of the current proof prefix, the frontier pruning discards the subsumed clauses and backtracks once any such clause also happens to be unused in the proof, and the dominance pruning compares the current proof prefix with the visited ones and backtracks if one of them is ``clearly worse'' than the other.

Our approach relies on a variety of observations and algorithmic routines that could be further improved with a substantial effect. One of the key improvement points is the bounding, as it relies on a SMUS solver as the source of the lower bounds, which suggests that a tighter integration of the SMUS solver into the branch-and-bound procedure would improve the runtime even further. On the other hand, Lemma~\ref{lemma:bounding-mus-enum} itself limits the tightness of a bound: as this bound cannot exceed the size of the formula, it cannot prune subproblems containing, for example, pigeonhole formulas. This indicates two alternative routes for improvement of the lower bound pruning, namely, deriving stronger relaxations than the one offered by Lemma~\ref{lemma:bounding-mus-enum} and detecting common sub-formulas with known superlinear bounds.

This approach can also be applied to proof systems that do not operate on propositional formulas. For example, \citeA{Sidorov2024-Other} have proposed a proof system that produces tree-like proofs of optimality in constrained shortest path problems. However, the approach for discovering short proofs in that system, while empirically performant, does not carry any form of optimality guarantees, suggesting an extension of the described approach to the shortest path (or other similar) proof systems. The main technical challenge in that line of work is in the bounding implementation, as it consumes the majority of time in the experimental runs as well as takes advantage of the shortest proof structure from Lemma~\ref{lemma:bounding-mus-enum}.

\section*{Acknowledgements}

Konstantin Sidorov is supported by the TU Delft AI Labs program as part of the XAIT lab.

\vskip 0.2in
\bibliography{references-final}
\bibliographystyle{theapa}

\appendix

\section{Proofs}
\label{appendix:proofs}

\begin{proof}[Proof of Lemma~\ref{lemma:branching}]
    Observe that \[ \Known{\mathcal{P}'_j} = \Known{\mathcal{P}} \cup \{ \omega_j \} \] and \[ \Known{\mathcal{P}^{commit}} = \Known{\mathcal{P}}. \] Suppose that \( L \) is a compatible proof of \( \mathcal{P} \) that starts with the clauses from \( \Known{\mathcal{P}} \); the remainder of the proof depends on whether \( L \) uses any \( \omega_j \) clauses.
    
    If no such clauses are produced in \( L \), then \( L \) is also a compatible proof for \( \mathcal{P}^{commit} \), because extending the forgotten set does not impact any further decisions. However, \( \Known{\mathcal{P}^{commit}} \) only differs from \( \Known{\mathcal{P}} \) by pushing some of the clauses ``backward". As some prefix of \( L \) agrees with \( \mathcal{P} \), it must also agree with \( \mathcal{P}^{commit} \).

    Otherwise, let \( \omega_j \) be the clause used in \( L \) with the smallest index \( j \); by take-it-or-leave-it property, that means that the clauses \( \omega_1, \dotsc, \omega_{j-1} \) are not used in the proof. Then \( L' \) is a compatible proof of \( \mathcal{P}_j \), as it starts with the same prefix (up to the current layer) and does not use clauses from the set \( \{ \omega_1, \dotsc, \omega_{j-1}, \omega_j \}\), which is exactly the definition of \(  \mathcal{P}_j \).
\end{proof}

\begin{proof}[Proof of Theorem~\ref{theorem:frontier-suffiency}]
    Let \( (Q_j = L_j \diamond R_j )_{j=1}^K \) be the suffix proof of \( F \) following the introduction of axioms. We describe how to construct a proof \( (Q'_j = L'_j \diamond R'_j)_{j=1}^K \) of \( \Frontier{F} \) such that \( Q'_j \subseteq Q_j \). As the last clause \( Q_K \) of the original proof is empty, so is the last clause \( Q'_K \) of the frontier proof.

    Given a clause \( L_j \) of the original proof, let \( U_j \subseteq L_j \) be a clause from the frontier proof included in the corresponding premise of the original proof. More specifically, \( U_j \fin F \) if \( L_j \in F \) and \( U_j = Q'_\ell\) if \( U_j = Q_\ell \) for \( \ell < j \). Similarly, let \( V_j \subseteq R_j \) be a corresponding clause for another premise of the same step. (Note that this definition only depends on the clauses from the original formula and earlier derivation steps, which leads to a proof by induction.)

    Suppose that the pivot literal is present in both premises \( U_j \) and \( V_j \); in that case, any literal in \( U_j \diamond V_j \) is contained either in \( U_j \subseteq L_j \) or in \( V_j \subseteq R_j \), meaning that \( U_j \diamond V_j \subseteq L_j \diamond R_j = Q_j \). Thus, in that case setting \( (L'_j, R'_j) = (U_j, V_j) \) is sufficient to ensure \( Q'_j \subseteq Q_j \).

    On the other hand, if the pivot literal is missing in (without loss of generality) \( U_j \), deriving \( Q'_j = U_j \) is sufficient to ensure \( Q'_j \subseteq Q_j \). If \( U_j \in_\mathcal{F} F \), then we can introduce it with a bogus derivation \( U_j \diamond U_j = U_j \). Otherwise, \( U_j \) has already been derived by the \( j \)-th step by definition of \( U \)-clauses, meaning that we can copy this derivation into the \( j \)-th step.
\end{proof}

\begin{proof}[Proof of Lemma~\ref{lemma:lookup}]
    Consider a compatible proof \( L' \) of \( \mathcal{P}' \); we show that there is a compatible proof \( L \) of \( \mathcal{P} \) that delivers a proof at most as long as \( L' \). We assume that the first \( k \) layers are fixed by the \( \Previous{\mathcal{P'}} \) and \( \Current{\mathcal{P'}} \) sets. Suppose that the \( \Next{\mathcal{P}'} \) corresponds to the \( k' \)-th layer of \( L' \); we start the construction by adding all the missing clauses from that layer into the \( k \)-th layer of \( L \), which is possible by Definition~\ref{def:dominance}.2.

    Next, we reproduce the suffix of the \( L' \), starting with \( (k'+1)\)-st layer, as the suffix of \( L \) -- that is possible by Definition~\ref{def:dominance}.3 (the dominating problem can derive any clause the dominated problem can, and using the same premises at that) and~\ref{def:dominance}.4 (if a clause has not been forgotten in the dominated problem, it cannot be forgotten in the dominating problem). This completes the construction, and it remains to show that the resulting proof is not worse than the original one.
    
    Our construction mirrors the clauses from the final part of \( L' \) to produce  \( L \), implying that this suffix \( L \) has the same clauses as the suffix of \( L' \). In addition, as it uses the same premises to derive the same resolvents, if \( L \) uses a new axiom, so does \( L' \) due to~\ref{def:dominance}.1; this, in turn, means that \( L \) does not use more axioms than \( L' \). Applying Definition~\ref{def:dominance}.5 shows that the remaining part of the proof between the axioms and ``new" clauses is at most as long in \( L \) as in \( L' \). Combining these three facts, we derive that \( L \) does not introduce more clauses than \( L' \), which is precisely the requirement of Proposition~\ref{proposition:correctness}.3.
\end{proof}

\begin{proof}[Proof of Lemma~\ref{lemma:bounding-mus-enum}]
    Observe that any compatible proof of \( \mathcal{F} \) can be seen as a proof of unsatisfiability of some subset of clauses \( S \fsubseteq \Known{\mathcal{P}} \), which means that the lowest bound among the bounds for \( S \fsubseteq \Known{\mathcal{P}}  \) is a correct bound for any ``remaining" proof length for \( \mathcal{P} \). Given that \emph{all} of the clauses in \( S \) are used in the proof, we can use the proof of Lemma~\ref{lemma:bounding-mus} to show that any such proof has length at least \( \#S - 1 \). We complete the proof by taking the most optimistic of discovered bounds, which corresponds to minimizing \( \#S - 1\) over UNSAT subsets of \( \Frontier{\Known{\mathcal{P}}} \), just as in the lemma statement.
\end{proof}

\begin{proof}[Proof of Lemma~\ref{lemma:bounding-mus-enum-fixed}]
    For an unused clause \( \omega \in U \), retrace the branching steps needed to reach the subproblem \( \mathcal{P} \) and forget \( \omega \) instead of deriving it. The resulting problem can be used as \( \mathcal{P}' \), and omitting the derivation for \( \omega \) yields its compatible proof \( L' \).
\end{proof}

\section{Algorithm Configuration Experiments}
\label{appendix:configuration}

To generate the dataset for configuring the heuristic parameters, we use the same procedure as used in the main text for the synthetic dataset with fewer generator runs and only include random 3-CNFs, subset cardinality, and graph coloring formulas. We do this to facilitate the algorithm parameter tuning without overfitting the collection used to report the results.

All the experiments in this section are executed on the same hardware as the experiments in the main text but with lower resource limits: 5 minutes and 12 GB of RAM per run. The only exception is Section~\ref{appendix:sub:cleanup}, as those experiments were run with 1-hour time limits; we had to raise the limits for this set of experiments since the 5-minute runs did not show any meaningful difference between the dominance cache clean-up strategies.

During our experiments, we observed that the proof length exhibits little variance with respect to the algorithm configuration. Our explanation of this is that the instances in the validation set are simple enough to accommodate many completion calls, with one of them being likely to yield a good proof regardless of the algorithm configuration. Therefore, to simplify our analysis, we only tune the algorithm for the time-focused runs and use the same configuration for length-focused runs except for the completion procedure (CaDiCaL) and clause ordering (ascending length).

Throughout the analysis, we evaluate different configurations based on the \emph{virtual best} time, which is defined for any instance from the validation sample as the fastest time to optimality among the considered configurations. More precisely, we compare the distributions of \emph{relative gaps} to virtual best, which we define for an experimental run as a ratio of the elapsed time to optimality (if applicable) to the virtual best time for that instance.

\subsection{Completion Backend and Randomization}

We start by comparing four configurations corresponding to two different completion backends (DPLL, CaDiCaL) and static/dynamic seeding. Figure~\ref{fig:config-backend-time} reports the results of this comparison, which suggests that the DPLL backend with dynamic seeding is the most robust configuration. Dynamic seeding appears to be helpful regardless of the completion backend as it helps diversify the proofs produced in the completion subroutine, thereby increasing the odds of discovering the optimal proof. Surprisingly, however, DPLL is consistently faster in proving optimality than CaDiCaL; we believe that the reason behind this is that if an \emph{optimal} proof can be established, then it is likely to be extremely simple, and the overhead of running CaDiCaL negates any advantage that could be gained from receiving high-quality proofs from it.

\begin{figure}
    \centering
    \includesvg[width=1.0\linewidth]{completion-shuffle-time}
    \caption{Distribution of relative gap to virtual best for completion backend and seeding strategy combinations.}
    \label{fig:config-backend-time}
\end{figure}

\subsection{Clause Ordering Strategy}

We evaluated two different ordering strategies that depend on sorting clauses by their length: \emph{ascending length} corresponds to sorting clauses from short to long, and \emph{descending length} corresponds to sorting clauses from long to short. As these strategies produce many ties, we also consider the \emph{ascending frequency} and \emph{descending frequency} tie-breaking strategies, which count the occurrences of a clause literal in the frontier set and give preference to respectively the least and the most occurrences.

\begin{table}[]
\begin{tabular}{|l|c|c|c|c|}
\hline
\textbf{Strategy}  & \makecell{Length ↑,\\frequency ↑} & \makecell{Length ↑,\\frequency ↓} & \makecell{Length ↓,\\frequency ↑} & \makecell{Length ↓,\\frequency ↓} \\ \hline
\textbf{Gap to the virtual best} & 498\%                            & 402\%                              & 322\%                              & 259\%                                \\ \hline
\end{tabular}
\caption{Relative gap to the virtual best time to optimality per sorting strategy; all values are averaged across all runs with geometric mean.}
\label{table:config-ord-time-vbs}
\end{table}

\begin{figure}[h!]
    \centering
    \includesvg[width=0.8\linewidth]{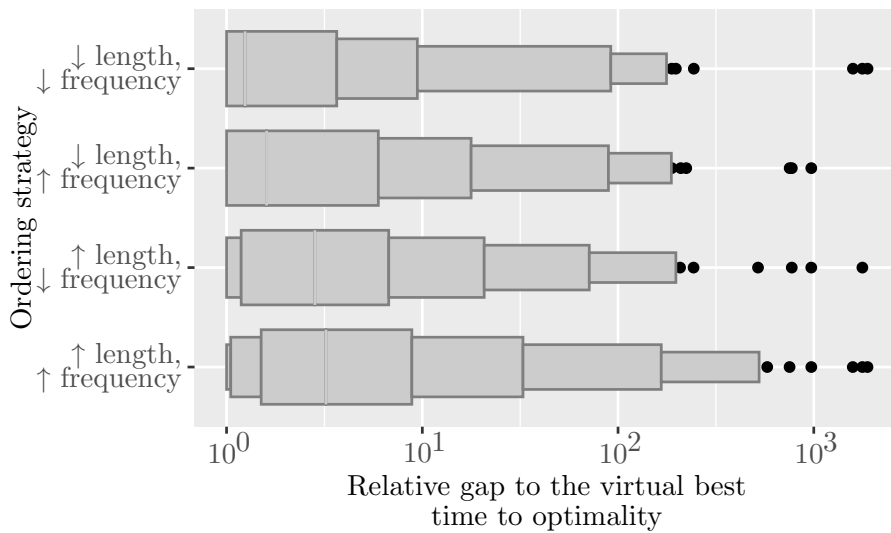}
    \caption{Distribution of relative time to optimality between ascending-length and descending-length ordering strategies.}
    \label{fig:config-ord-time}
\end{figure}

Table~\ref{table:config-ord-time-vbs} compares the gaps to the virtual best time to optimal, suggesting the descending-length, descending-frequency strategy as the clear leader. Figure~\ref{fig:config-ord-time} supports this statement and highlights the robustness of this choice, with all the illustrated percentiles being better for the descending-length, descending-frequency strategy. Our explanation for these results is that starting with longer clauses is helpful because that means that the longest clauses are deleted in most of the produced subproblems, in extreme cases even being immediately subsumed by some other clause.

\subsection{SMUS Switch-point}

\begin{table}[]
\begin{tabular}{|l|c|c|c|c|c|c|}
\hline
\textbf{Switch-point}  & 16 & 20 & 24 & 28 & 32 & 36 \\ \hline
\textbf{Gap to the virtual best} & 862\% & 529\% & 437\% & 377\% & 478\% & 500\% \\ \hline
\end{tabular}
\caption{Relative gap to the virtual best time to optimality; all values are averaged across all runs with geometric mean.}
\label{table:switch-point}
\end{table}

\begin{figure}
    \centering
    \includesvg[width = 0.8\linewidth]{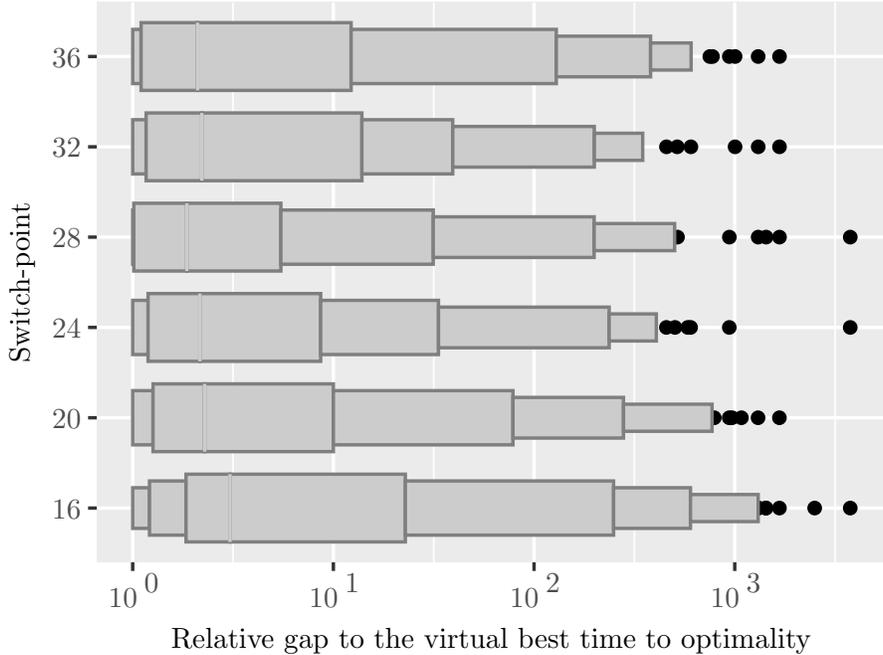}
    \caption{Distribution of relative gaps to the virtual best times per switch point value.}
    \label{fig:switch-point}
\end{figure}

We considered the switch-point values \( M_{switch} \in \{ 16, 20, 24, 28, 32 \} \), and, as before, we compare the time to optimality with the virtual best time for that instance. Table~\ref{table:switch-point} aggregates the geometric means of relative times to optimality for each of the \( M_{switch} \) values and suggests using \( M_{switch} = 28 \) as the switch-point value performing best on average. This suggestion is also supported by Figure~\ref{fig:switch-point}, which underlines that most of the quantiles are also the best for \( M_{switch} = 28 \), except for the highest ones. 

\subsection{Dominance Cache Clean-up}
\label{appendix:sub:cleanup}

\begin{table}[]
\begin{tabular}{|l|c|c|c|c|}
\hline
\textbf{Cache entry lifetime \( T_{\max} \)} & \( +\infty \) & \( 10^3 \) & \( 10^4 \) & \( 10^5 \) \\ \hline
\textbf{Gap to the virtual best} & 508\% & 344\% & 372\% & 264\% \\ \hline
\end{tabular}
\caption{Relative gap to the virtual best time to optimality; all values are averaged across all runs with geometric mean.}
\label{table:dominance}
\end{table}

We consider discarding the entries that have not been accessed for the last \( T_{\max} \) branch-and-bound loop iteration; to this end, we evaluate the performance on the grid of values \( T_{\max} \in \{ 10^3, 10^4, 10^5, +\infty \} \). Table~\ref{table:dominance} suggests that this is a reasonable strategy for reducing the memory footprint; while all considered values of \( T_{\max} \) exhibit improvement over the default strategy, we chose \( T_{\max} = 10^5 \) due to its better average performance and more consistent behavior, as indicated by Figure~\ref{fig:dominance}.

\begin{figure}
    \centering
    \includesvg[width = 0.8\textwidth]{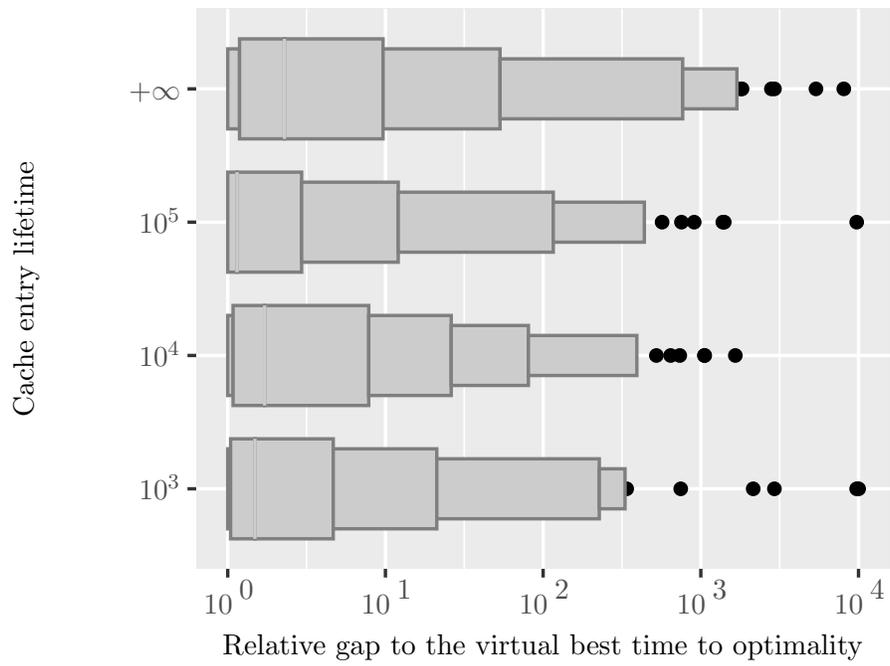}
    \caption{Distribution of relative gaps to the virtual best times per \( T_{\max} \) value.}
    \label{fig:dominance}
\end{figure}

\end{document}